\documentclass[12pt]{article}
\usepackage{natbib}
\usepackage{url}
\usepackage[colorlinks,citecolor=blue,urlcolor=blue]{hyperref}
\usepackage{booktabs}
\usepackage{amsfonts}
\usepackage{nicefrac}
\usepackage{xcolor}
\usepackage{bm}
\usepackage{amsthm,amsmath,amsopn,amssymb}
\usepackage{verbatim}
\usepackage[noend]{algpseudocode}
\usepackage{algorithmicx, algorithm}
\usepackage{graphicx}
\usepackage{threeparttable}
\usepackage{multicol, multirow}
\usepackage{enumitem}
\usepackage{wrapfig}
\usepackage{subfigure}
\usepackage{accents}

\newcommand{\blind}{0}

\newtheorem{corollary}{Corollary}
\newtheorem{theorem}{Theorem}
\newtheorem{lemma}{Lemma}

\theoremstyle{definition}

\makeatletter
\def\widebar{\accentset{{\cc@style\underline{\mskip10mu}}}}
\makeatother

\newcommand{\T}{\mathbf{T}}
\newcommand{\C}{\mathbf{C}}
\newcommand{\K}{\mathbf{K}}

\newcommand{\M}{\mathbf{M}}

\newcommand{\cL}{\mathcal{L}}
\newcommand{\cC}{\mathcal{C}}
\newcommand{\cR}{\mathcal{R}}
\newcommand{\cS}{\mathcal{S}}
\newcommand{\cE}{\mathcal{E}}
\newcommand{\tL}{\widetilde{\mathbf{L}}}
\newcommand{\tK}{\widetilde{\mathbf{K}}}
\newcommand{\tT}{\widetilde{\mathbf{T}}}
\newcommand{\tcC}{\widetilde{\mathcal{C}}}
\newcommand{\barK}{\widebar{\mathbf{K}}}

\def\u{\bm{u}}
\def\v{\bm{v}}
\def\a{\bm{a}}
\def\b{\bm{b}}

\def\eps{\varepsilon}
\def\lam{\lambda}

\newcommand{\RR}{\mathbb{R}}

\newcommand{\W}{\operatorname{W}}
\newcommand{\GW}{\operatorname{GW}}
\newcommand{\diag}{\operatorname{diag}}

\usepackage[left=2.5cm,right=2.5cm,top=2.5cm,bottom=2.5cm]{geometry}

\begin{document}

\def\spacingset#1{\renewcommand{\baselinestretch}%
{#1}\small\normalsize} \spacingset{1}


\if0\blind
{
\title{\bf Efficient Approximation of Gromov-Wasserstein Distance Using Importance Sparsification}
\author{
Mengyu Li \\
Institute of Statistics and Big Data, Renmin University of China \\
Jun Yu\thanks{Joint first author} \\
School of Mathematics and Statistics, Beijing Institute of Technology \\
Hongteng Xu \\
Gaoling School of Artificial Intelligence, Renmin University of China \\
Cheng Meng\thanks{Corresponding author, chengmeng@ruc.edu.cn} \\
Center for Applied Statistics,\\
Institute of Statistics and Big Data, Renmin University of China \\
}
\date{}
  \maketitle
} \fi

\if1\blind
{
  \bigskip
  \bigskip
  \bigskip
  \begin{center}
    {\LARGE\bf Efficient Approximation of Gromov-Wasserstein Distance Using Importance Sparsification}
\end{center}
  \medskip
} \fi

\bigskip
\begin{abstract}
As a valid metric of metric-measure spaces, Gromov-Wasserstein (GW) distance has shown the potential for matching problems of structured data like point clouds and graphs. However, its application in practice is limited due to the high computational complexity. To overcome this challenge, we propose a novel importance sparsification method, called \textsc{Spar-GW}, to approximate GW distance efficiently. In particular, instead of considering a dense coupling matrix, our method leverages a simple but effective sampling strategy to construct a sparse coupling matrix and update it with few computations. The proposed \textsc{Spar-GW} method is applicable to the GW distance with arbitrary ground cost, and it reduces the complexity from $O(n^4)$ to $O(n^{2+\delta})$ for an arbitrary small $\delta>0$. Theoretically, the convergence and consistency of the proposed estimation for GW distance are established under mild regularity conditions. In addition, this method can be extended to approximate the variants of GW distance, including the entropic GW distance, the fused GW distance, and the unbalanced GW distance. Experiments show the superiority of our \textsc{Spar-GW} to state-of-the-art methods in both synthetic and real-world tasks.
\end{abstract}

\noindent%
\vfill

\newpage
\spacingset{1.5} 

\section{Introduction}\label{intro}
Gromov-Wasserstein (GW) distance, as an extension of classical optimal transport distance, is originally proposed to measure the distance between different metric-measure spaces~\citep{memoli2011gromov,sturm2006geometry}. 
Recently, it attracts wide attention due to its potential for tackling challenging machine learning tasks, including but not limited to shape matching~\citep{memoli2011gromov,ezuz2017gwcnn,titouan2019sliced}, graph analysis~\citep{chowdhury2019gromov,xu2019gromov,titouan2019optimal,chowdhury21generalized,brogat2022learning,vincent2022semi,xu2023representing}, point cloud alignment~\citep{peyre2016gromov,alvarez2018gromov,alaux2019unsupervised,blumberg2020mrec}, and distribution comparison across different spaces~\citep{yan2018semi,bunne2019learning,chapel2020partial,gong2022gromov}.

Despite the wide application, calculating the GW distance is NP-hard, which corresponds to solving a non-convex non-smooth optimization problem. 
To bypass this obstacle, many efforts have been made to approximate the GW distance with low complexity. 
One major strategy is applying the conditional gradient algorithm (or its variants) to solve the GW distance iteratively in an alternating optimization framework~\citep{peyre2016gromov,titouan2019optimal}. 
By introducing an entropic regularizer~\citep{solomon2016entropic} or a proximal term based on the Bergman divergence~\citep{xu2019gromov}, the subproblem in each iteration will be strictly convex and can be solved via the Sinkhorn-scaling algorithm~\citep{sinkhorn1967concerning, cuturi2013sinkhorn}. 
Another strategy is computing sliced Gromov-Wasserstein distance~\citep{titouan2019sliced}, which projects the samples to different 1D spaces and calculates the expectation of the GW distances defined among the projected 1D samples.
More recently, to further reduce the computational complexity, more variants of the GW distance have been proposed, which achieve acceleration via imposing structural information (e.g., tree~\citep{le2021flow}, low-rank~\citep{xu2019scalable,sato2020fast,chowdhury2021quantized}, and sparse structure~\citep{xu2019scalable}) on the ground cost $\cL$, the coupling matrix $\T$, or both~\citep{scetbon2022linear}. 
However, most of these methods mainly focus on the GW distance using decomposable ground cost functions~\citep{peyre2016gromov}. 
Moreover, some of them are only applicable for specific data types (e.g., point clouds in Euclidean space~\citep{titouan2019sliced} and sparse graphs with clustering structures~\citep{xu2019scalable, blumberg2020mrec}), and they are not able to approximate the original GW distance. 
See Table~\ref{tab:method-compare} for an overall comparison.
Therefore, it is urgent to develop a new approximation of the GW distance that has better efficiency and applicability.

{\spacingset{1.25}
\begin{table}
  \caption{Comparison for various GW distance approximation methods on their time complexity, assumptions imposed on their ground cost functions and coupling matrices, and required data types.}
  \label{tab:method-compare}
  \centering
  \resizebox{\textwidth}{!}{%
  \begin{threeparttable}
  \begin{tabular}{lccccc}
    \toprule
    Method
    &Time $O(\cdot)$ 
    &Ground cost $\cL$
    &Coupling $\T$
    &Data type  \\
    \midrule
    Entropic GW~\citep{peyre2016gromov} 
    &$n^3$ 
    &Decomposable
    &-
    &-\\
    Sliced GW~\citep{titouan2019sliced} 
    &$n^2$ 
    &$\ell_2$ loss
    &-
    &Points\\
    S-GWL~\citep{xu2019scalable} 
    &$n^2\log(n)$  
    &Decomposable
    &Sparse \& low-rank
    &-\\
    AE~\citep{sato2020fast} 
    &$n^2\log(n^2)$ 
    &$p$-power ($p\in\mathbb{Z}_{+}$)
    &-
    &-\\
    FlowAlign~\citep{le2021flow} 
    &$n^2$
    &$\ell_2$ loss
    &Tree-structure  
    &Points\\
    Linear-time GW~\citep{scetbon2022linear} 
    &$r_2(r_1+r_2)n$ 
    &$\ell_2$ loss
    &Low-rank
    &Points\\
    SaGroW~\citep{kerdoncuff2021sampled} 
    &$n^2(s^\prime+\log(n))$  
    &-
    &-
    &-\\
    \textsc{Spar-GW} (Proposed) 
    &$n^2+s^2$
    &-
    &-
    &-\\
    \bottomrule
  \end{tabular}
  \begin{tablenotes}
  \item[1] \footnotesize{``-'' means no constraints. For the column of data type, ``-'' means the data can be sample points and/or their relation matrices.}
  \item[2] \footnotesize{For the column of time complexity, the time of calculating relation matrices of points is also included.}
  \item[3] \footnotesize{$n$ represents the sample size. For Linear-time GW, $r_1$ and $r_2$ are the assumed ranks of relation matrices and coupling matrix, respectively; for SaGroW, $s^\prime$ denotes the number of sampled matrix; for \textsc{Spar-GW}, $s$ denotes the number of sampled elements.}
  \end{tablenotes}
  \end{threeparttable}
  }
\end{table}
}

\textbf{Major contribution.}
In this paper, we propose a randomized sparsification method, called \textsc{Spar-GW}, to approximate the GW distance and its variants.
In particular, during the iterative optimization of the GW distance, the proposed \textsc{Spar-GW} method leverages an importance sparsification mechanism to derive a sparse coupling matrix and the corresponding kernel matrix.
Replacing dense multiplications with sparse ones leads to an efficient approximation of GW distance with $O(n^2 + s^2)$ time complexity, where $s$ is the number of selected elements from an $n\times n$ kernel matrix. 
The \textsc{Spar-GW} is compatible with various computational methods, including the proximal gradient algorithm for original GW distance and the Sinkhorn-scaling algorithm for entropic GW distance, and it is capable of arbitrary ground cost.
In theory, we show the proposed estimator is asymptotically unbiased when $s=O(n^{1+\delta})$ for an arbitrary small $\delta>0$, under some regularity conditions. 
Table~\ref{tab:method-compare} highlights the advantage of our method. 
Moreover, this method can be extended to compute the variants of GW distance, e.g., a straightforward application to approximate the fused GW (FGW) distance and a non-trivial extension called \textsc{Spar-UGW} to approximate the unbalanced GW (UGW) distance.
Experiments show the superiority of the proposed methods to state-of-the-art competitors in both synthetic and real-world tasks.

The remainder of this paper is organized as follows. We start in Section~\ref{sec:background} by introducing computational optimal transport and GW distance. 
In Section~\ref{sec:method}, we develop the sampling probabilities and provide details of our main algorithm.
The theoretical properties of the proposed estimator are presented in Section~\ref{sec:theory}.
Section~\ref{sec:spar-ugw} extends the proposed method to unbalanced problems.
We examine the performance of the proposed algorithms through extensive synthetic and real-world datasets in Section~\ref{sec:experiments}.
Technical proofs and more experimental results are provided in the Appendix.

\section{Background}\label{sec:background}

In the following, we adopt the common convention of using uppercase boldface letters for matrices, lowercase boldface letters for vectors, and regular font for scalars. 
We denote non-negative real numbers by $\mathbb{R}_{+}$ and the set of integers $\{1, \ldots, n\}$ by $[n]$. 
We use $\mathbf{1}_n$ and $\mathbf{0}_n$ to denote the all-ones and all-zeroes vectors in $\mathbb{R}^{n}$, respectively. For a matrix $\mathbf{A}=(A_{i j})$, its spectral norm (i.e., the largest singular value) and Frobenius norm are denoted as $\|\mathbf{A}\|_2$ and $\|\mathbf{A}\|_F$, respectively. The condition number of $\mathbf{A}$ is defined as $\|\mathbf{A}\|_2/\sigma_{\min}(\mathbf{A})$, where $\sigma_{\min}(\cdot)$ stands for the smallest singular value.
We denote by $\exp (\mathbf{A})$ the matrix with entries $\exp (A_{ij})$. 
For two matrices $\mathbf{A}$ and $\mathbf{B}$ of the same dimension, we denote their Frobenius inner product by $\langle \mathbf{A}, \mathbf{B} \rangle=\sum_{i,j} A_{i j} B_{i j}$.

\subsection{Computational optimal transport} 
Consider two samples $X=\{\bm{x}_i\}_{i=1}^{m}$ and $Y=\{\bm{y}_j\}_{j=1}^{n}$ that are generated from the distributions $\a\in\Delta^{m-1}$ and $\b\in\Delta^{n-1}$, respectively, where $\Delta^{n-1}$ represents the $(n-1)$-Simplex.
When $\a$ and $\b$ lie in the same space, optimal transport (OT) and its associated Wasserstein distance~\citep{villani2009optimal} are used extensively to quantify the discrepancy between these two probability distributions.
The modern Kantorovich formulation~\citep{kantorovich1942transfer} of OT writes
\begin{equation}\label{eq:ot}
\W(\a, \b):=\min_{\T \in \Pi(\a, \b)}\langle \M,\T \rangle,
\end{equation}
where $\M \in \mathbb{R}^{m \times n}$ is a given distance matrix, $\Pi(\a, \b) = \{\T \in \RR_{+}^{m \times n}: \T \mathbf{1}_{n} = \a, \T^{\top} \mathbf{1}_{m} = \b\}$ is the set of admissible coupling matrices, i.e., all joint probability distributions with marginals $\a, \b$, and the $(i,j)$-th entry of $\T$ represents the amount of probability mass shifted from $i$ to $j$.
The solution to~\eqref{eq:ot} is called the optimal transport plan.
If $\M$ is a distance matrix of order $p$, $\W_p(\cdot, \cdot)=\W(\cdot, \cdot)^{1/p}$ is called the $p$-Wasserstein distance.

Despite the wide applications, the computational complexity of directly solving~\eqref{eq:ot} using a linear program grows cubically as $m$ or $n$ increases~\citep{brenier1997homogenized, benamou2002monge}.
To approximate the optimal transport plan efficiently, \citet{cuturi2013sinkhorn} added an entropic regularization term on~\eqref{eq:ot}, which leads to a strongly convex and smooth problem
\begin{equation}\label{eq:rot}
\min_{\T \in \Pi(\a, \b)}\langle \M,\T \rangle + \eps H(\T),
\end{equation}
where $\eps>0$ is a regularization parameter and $H(\T)=\langle \T,\log\T\rangle$ is the negative Shannon entropy of $\T$.
By introducing a kernel matrix $\K := \exp(-\M/\eps)$, it is known that the solution to~\eqref{eq:rot} is a projection onto $\Pi(\a, \b)$ of $\K$~\citep{peyre2019computational}. Therefore, the problem~\eqref{eq:rot} can be solved by using iterative matrix scaling~\citep{sinkhorn1967concerning}, called the Sinkhorn-scaling algorithm~\citep{cuturi2013sinkhorn}; see Step~5 in Algorithm~\ref{alg:gw} for details.
The Sinkhorn-scaling algorithm enables researchers to approximate the OT solution efficiently, and thus has been extensively studied in recent years \citep{genevay2019sample, lin2019acceleration, scetbon2020linear, liao2022fast, liao2022fast2}.
There also exist projection-based methods to approximate the Wasserstein distance \citep{bonneel2015sliced, meng2019large, deshpande2019max, li2022hilbert}. We refer to \cite{nadjahi2021sliced} and \cite{zhang2022projection} for recent reviews.

\subsection{Computational GW distance} 
When $\a$ and $\b$ lie in different spaces, the distance matrix $\M$ is unavailable and thus optimal transport can not be used.
As a replacement, the Gromov-Wasserstein distance is applicable to measuring the discrepancy between two samples located in different sample spaces by comparing their structural similarity.
Similar to Wasserstein distance, the intuition of GW distance is still to minimize the transportation effort, but GW only relies on the structure in each space separately.

In particular, consider two relation matrices $\C^X=(C_{i i^{\prime}}^{X})\in\mathbb{R}^{m\times m}$ and $\C^Y=(C_{j j^{\prime}}^{Y})\in\mathbb{R}^{n\times n}$, each of which can be the distance/kernel matrix defined on a sample~\citep{memoli2011gromov,peyre2016gromov}, or the adjacency matrix of a graph constructed by the sample~\citep{xu2019gromov,titouan2019optimal}.
Let $\cL: \RR \times \RR \mapsto \RR$ be the ground cost function, e.g., the $\ell_1$ loss (i.e., $\cL(x_1,x_2)=|x_1-x_2|$), the $\ell_2$ loss (i.e., $\cL(x_1,x_2)=(x_1-x_2)^2$), and the Kullback-Leibler (KL) divergence (i.e., $\cL(x_1,x_2)=x_1 \log (x_1 / x_2)-x_1+x_2$).
The GW distance is defined as the following non-convex non-smooth optimization problem:
\begin{eqnarray}\label{eq:gw}
\begin{aligned}
\GW \left((\C^X, \a), (\C^Y, \b)\right) :=& \sideset{}{_{\T \in \Pi(\a, \b)}}\min \sideset{}{_{i, i^{\prime}, j, j^{\prime}}}\sum \cL\left(C_{i i^{\prime}}^{X}, C_{j j^{\prime}}^{Y}\right) T_{i j} T_{i^{\prime} j^{\prime}}\\
=&\sideset{}{_{\T \in \Pi(\a, \b)}}\min \langle \cL(\C^X, \C^Y)\otimes \T, \T \rangle=\langle \cL(\C^X, \C^Y)\otimes \T^*, \T^* \rangle,
\end{aligned}
\end{eqnarray}
where $\cL(C_{i i^{\prime}}^{X}, C_{j j^{\prime}}^{Y})$ measures similarity between pairs of points or graphs, $T_{i j}$ is the $(i,j)$-th entry of the coupling matrix $\T$, and then the term $\cL\left(C_{i i^{\prime}}^{X}, C_{j j^{\prime}}^{Y}\right) T_{i j} T_{i^{\prime} j^{\prime}}$ represents the transport cost between two pairs $(i,i^{\prime})$ and $(j,j^{\prime})$.
As shown in~\eqref{eq:gw}, this optimization problem can be written in a matrix format~\citep{peyre2016gromov}, where $\cL(\C^X, \C^Y)$ is a tensor, and $\cL(\C^X, \C^Y)\otimes \T:= [(\sum_{i^\prime,j^\prime} \cL(C^X_{i i^\prime}, C^Y_{j j^\prime}) T_{i^\prime j^\prime})_{i,j}]\in\mathbb{R}^{m\times n}$ is a tensor-matrix multiplication. 
As before, $\Pi(\a, \b)$ is the set of admissible coupling matrices, and the solution to problem~\eqref{eq:gw}, denoted as $\T^\ast$, is called the optimal transport plan.

In general, this optimization problem can be solved in an iterative framework: at the $r$-th iteration, the coupling matrix $\T$ is updated via solving the following subproblem:
\begin{eqnarray}\label{eq:iter}
    \T^{(r+1)}:=\arg\sideset{}{_{\T \in \Pi(\a, \b)}}\min \langle \underbrace{\cL(\C^X, \C^Y)\otimes \T^{(r)}}_{\cC(\T^{(r)})}, \T \rangle + \underbrace{\eps \cR(\T)}_{\text{Optional Reg.}}.
\end{eqnarray}
Here, $\cC(\T^{(r)})$ is a cost matrix determined by the previous coupling matrix $\T^{(r)}$, $\cR(\T)$ is an optional regularizer of $\T$, whose significance is controlled by the weight $\eps\geq 0$. 
The subproblem is essentially the (regularized) optimal transport problem~\eqref{eq:ot} or~\eqref{eq:rot}.
Without $\cR(\T)$, the problem in~\eqref{eq:iter} becomes a constrained linear programming given $\T^{(r)}$, and this is often solved via the conditioned gradient followed by line-search~\citep{titouan2019optimal}.
To improve the efficiency of solving the problem, \cite{xu2019gromov} implements $\cR(\T)$ as a Bregman proximal term, i.e., the KL-divergence $\text{KL}(\T \|\T^{(r)})=\langle\T,\log\T-\log\T^{(r)} \rangle$, which leads to a proximal gradient algorithm (PGA) and improves the smoothness of $\T$'s update. 
When the regularizer is implemented as the entropy of $\T$, i.e., $\cR(\T):=H(\T)$, the problem in~\eqref{eq:iter} becomes an entropic optimal transport problem and solving it iteratively leads to the approximation of GW distance or entropic GW distance~\citep{peyre2016gromov}.
Note that, when using the Bregman proximal term or the entropic regularizer, the problem in~\eqref{eq:iter} can be solved via the Sinkhorn-scaling algorithm~\citep{sinkhorn1967concerning,cuturi2013sinkhorn}, and accordingly, Algorithm~\ref{alg:gw} shows the computational scheme of the GW distance, where $\odot$ and $\oslash$ represent element-wise multiplication and division, respectively. When $\cR(\T)=H(\T)$, Algorithm~\ref{alg:gw} can also be used to compute the entropic GW distance by modifying the output to $\operatorname{GW}_{\eps} = \langle\cC(\T^{(R)}),\T^{(R)}\rangle +\eps H(\T^{(R)})$.

{\spacingset{1.25}
\begin{algorithm}[ht]
\caption{Computation of GW distance}
\begin{algorithmic}[1]
\State {\bf Input:} Sample distributions $\a, \b$, relation matrices $\C^{X}, \C^{Y}$, ground cost function $\cL$, regularization parameter $\eps$, number of outer/inner iterations $R$, $H$
\State Initialize $\T^{(0)} = \a \b^\top$
\State \textbf{For} $r=0 \textbf{ to } R-1$:
\State~~\textbf{Construct a kernel matrix:} \hfill \textcolor{red}{$O(m^2n^2)$}
\begin{itemize}
    \item[a)] Compute the cost matrix $\cC(\T^{(r)}) = \cL(\C^X, \C^Y)\otimes \T^{(r)}$
    \item[b)]
    $\K^{(r)}=
    \begin{cases} 
    \exp(-\frac{\cC(\T^{(r)})}{\eps})\odot \T^{(r)} & \text { if }\cR(\T)=\text{KL}(\T\|\T^{(r)}) \\
    \exp(-\frac{\cC(\T^{(r)})}{\eps}) & \text { if } \cR(\T)=H(\T)
    \end{cases}$
\end{itemize}
\State~~\textbf{Sinkhorn-scaling:} $\T^{(r+1)} = \textsc{Sinkhorn} (\a, \b, \K^{(r)}, H)$ \hfill \textcolor{red}{$O(Hmn)$}
\begin{itemize}
    \item[a)] Initialize $\u^{(0)} = \mathbf{1}_m, \v^{(0)} = \mathbf{1}_n$ 
    \item[b)] \textbf{For} $h=0 \textbf{ to } H-1$:\quad $\u^{(h+1)} = \a \oslash (\K^{(r)} \v^{(h)})$,~~$\v^{(h+1)} = \b \oslash (\K^{^{(r)}\top} \u^{(h+1)})$
    \item[c)] $\T^{(r+1)} = \operatorname{diag}(\u^{(H)}) \K^{(r)}  \operatorname{diag}(\v^{(H)})$
\end{itemize}
\State {\bf Output:} $\operatorname{GW} = \langle\cC(\T^{(R)}),\T^{(R)}\rangle$ \hfill \textcolor{red}{$O(m^2n^2)$}
\end{algorithmic}
\label{alg:gw}
\end{algorithm}
}

\subsection{Problem statement} 
The computational bottleneck of Algorithm~\ref{alg:gw} is the computation of the cost matrix $\cC(\T)$, which involves a tensor-matrix multiplication (i.e., the weighted summation of $mn$ matrices of size $m\times n$) with time complexity $O(m^2n^2)$.
Although the complexity can be reduced to $O(n^2m+m^2n)$ when the ground cost $\cL$ is decomposable, i.e., $\cL$ can be decomposed as $\cL(x_1, x_2) = f_1(x_1) + f_2(x_2) - h_1(x_1)h_2(x_2)$ for functions $(f_1, f_2, h_1, h_2)$, like the $\ell_2$ loss or the KL-divergence \citep{peyre2016gromov}. 
However, this setting restricts the choice of the ground cost and thus is inapplicable for GW distances in more general scenarios.

\section{Importance sparsification for GW distance}\label{sec:method}

\subsection{Importance sparsification} 
According to the analysis above, to approximate the GW distance efficiently, the key point is constructing a sparse $\tcC(\T)$ with low complexity as a surrogate for $\cC(\T)$, which motivates us to propose the \textsc{Spar-GW} algorithm. 
Replacing the $\cC(\T)$ with a sparse $\tcC(\T)$ results in two benefits.
First, $\tcC(\T)$ is associated with a sparse kernel matrix $\tK$, which enables us to use sparse matrix multiplications to accelerate the Sinkhorn-scaling algorithm, and the output is a sparse transport plan $\tT$ with the same sparsity structure as $\tK$, i.e., $\widetilde{T}_{ij}=0$ if $\widetilde{K}_{ij}=0$, as shown in Figure~\ref{fig:core-a}.
Second, when $\tT$ is sparse, $\tcC(\tT)$ can be calculated by summing $s<mn$ sparse matrices instead of $mn$ dense ones, and each of these sparse matrices contains at most $s$ non-zero elements, as shown in Figure~\ref{fig:core-b}.
Therefore, the principle of our \textsc{Spar-GW} algorithm is leveraging a simple but effective importance sparsification mechanism (i.e., constructing the sampling probability matrix $\mathbf{P}$ in Figure~\ref{fig:core-a}) to derive an informative sparse $\tcC(\tT)$, and accordingly, achieving an asymptotically unbiased estimate of the GW distance.

Recall that the GW distance in~\eqref{eq:gw} can be rewritten as a summation $\mbox{GW} =\sum_{i,j} T^\ast_{ij} \cC_{ij}^*$, where $\cC^*_{ij}$ is the $(i,j)$-th element of $\cC(\T^*)$.
According to the idea of importance sampling~\citep{liu1996metropolized,liu2008monte}, this summation can be approximated by a weighted sum of $s$ components, i.e, $\mbox{GW} \approx \sum_{(i,j)\in\mathcal{S}} T^\ast_{ij} \cC_{ij}^*/(sp_{ij})$, where $\mathcal{S}:= \{(i_l, j_l)\}_{l=1}^s$ represents the set of $s$ pairs of indices selected by the sampling probability $\{p_{ij}\}_{(i,j)\in[m] \times [n]}$.
Ideally, the optimal sampling probability, which leads to the minimum estimation variance, satisfies $p_{ij}^* \propto T^\ast_{ij} \cC_{ij}^*$.
Because neither the optimal $T^\ast_{ij}$ nor $\cC_{ij}^*$ is known beforehand, we propose to utilize a proper upper bound for $T^\ast_{ij} \cC_{ij}^*$ as a surrogate.
In particular, for the cost $\cC_{ij}^*$, we impose a mild assumption on it: $\exists c_0>0$ such that $\forall i,j$, $\cC_{ij}^*\leq c_0$.
Moreover, based on the constraint that $\T^*\in\Pi(\a,\b)$, we have $T^*_{ij}\leq a_i$ and $T^*_{ij}\leq b_j$, and thus $T^*_{ij}\leq \sqrt{a_i b_j}$.
Combining these inequalities, we utilize the sampling probability as
\begin{equation}\label{eq:pij-ot}
T^*_{ij}\cC_{ij}^*\leq c_0\sqrt{a_i b_j}\quad\Rightarrow\quad 
p_{ij} =\frac{\sqrt{a_i b_j}}{\sum_{i,j} \sqrt{a_i b_j}}, \quad 1\leq i\leq m,\quad 1\leq j\leq n.
\end{equation}
Intuitively, $T^*_{ij}$ can be large when both $a_i$ and $b_j$ are relatively large; otherwise, $T^*_{ij}$ should be small if either $a_i$ or $b_j$ is small. Therefore, from the perspective of importance sampling, it is natural to take the geometric mean of $a_i$ and $b_j$ as our sampling probability.

\begin{figure}[t]
    \centering
    \subfigure[The principle of \textsc{Spar-GW}.]{
    \includegraphics[height=4.7cm]{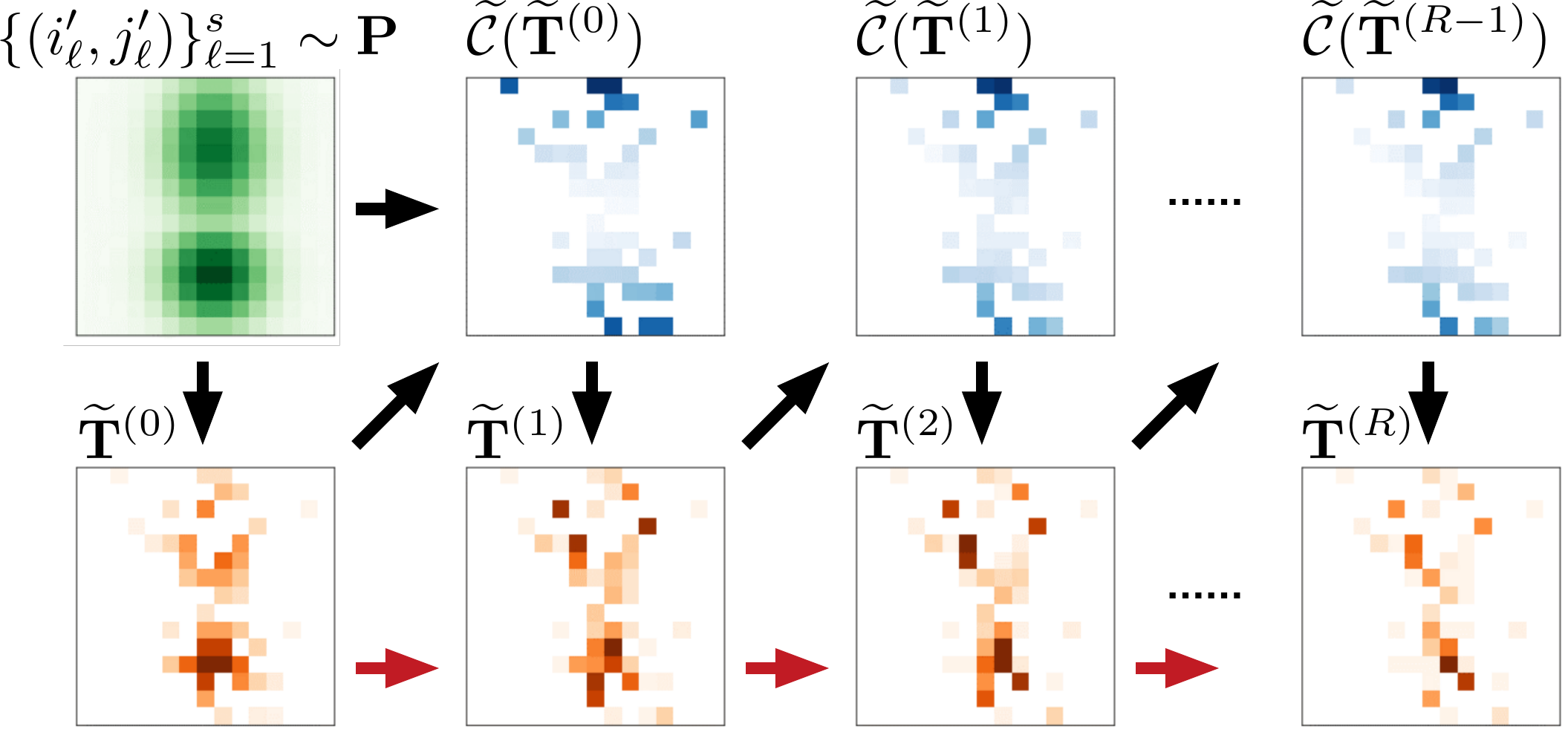}\label{fig:core-a}
    }
    \subfigure[The computation of $\tcC(\tT^{(r)})$.]{
    \includegraphics[height=4.7cm]{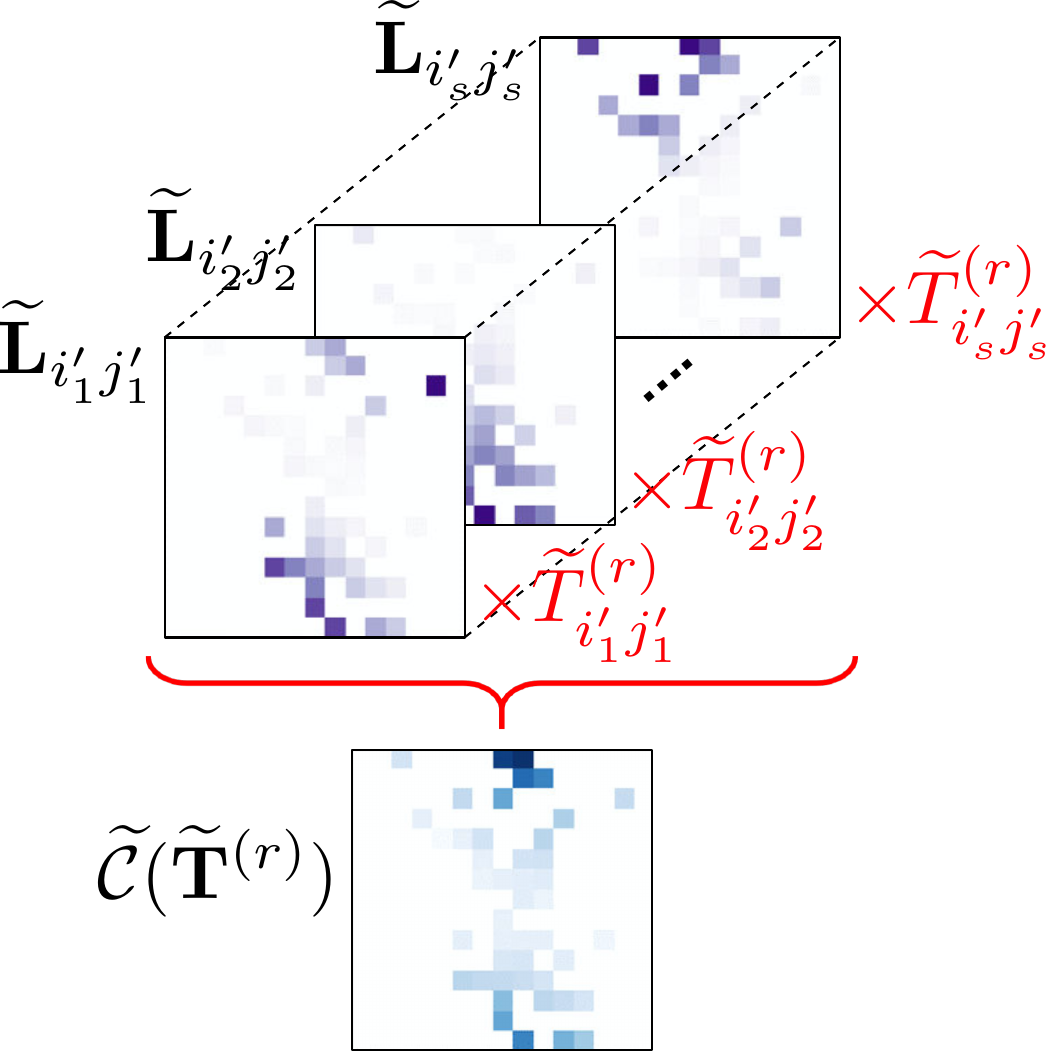}\label{fig:core-b}
    }
    \caption{(a) An illustration of our \textsc{Spar-GW}. 
    The non-zero elements of each matrix are labeled with colors. 
    The arrows between $\tT^{(r)}$'s are applied when using the proximal gradient algorithm. (b) An illustration of the computation of $\tcC(\tT^{(r)})$ in the $r$-th iteration.}
    \label{fig:core}
\end{figure}

\subsection{Proposed algorithm} 
Let $\mathbf{P}$ be the sampling probability matrix such that the $(i,j)$-th element equals the $p_{ij}$ in~\eqref{eq:pij-ot}.
Given $\mathbf{P}$, we first construct the index set $\mathcal{S}$ by sampling $s$ pairs of indices, and then, build $s$ matrices $\{\tL_{i\cdot j\cdot}\}_{(i,j)\in \mathcal{S}}$, whose elements are
\begin{eqnarray}\label{eq:spargw-tL}
\widetilde{L}_{ii^\prime jj^\prime} = 
\begin{cases} \cL(C^X_{ii^\prime}, C^Y_{jj^\prime}) & \text { if } (i^\prime,j^\prime)\in \mathcal{S} \\
0 & \text { otherwise. }
\end{cases}  
\quad \text{ for }(i,j)\in \mathcal{S}.
\end{eqnarray}
As shown in Figure~\ref{fig:core-b}, in the $r$-th iteration, we construct a sparse coupling matrix $\tT^{(r)}$, with $\widetilde{T}^{(r)}_{ij}=0$ if $(i,j)\notin \mathcal{S}$, and compute the sparse cost matrix $\tcC(\tT^{(r)}) = \sum_{(i,j)\in\mathcal{S}}\tL_{i\cdot j\cdot} \widetilde{T}^{(r)}_{ij}$. 
Accordingly, we derive the sparse kernel matrix $\tK^{(r)}$ with non-zero elements 
$$\widetilde{K}^{(r)}_{ij} = \exp(-\tcC_{ij}/\eps) \widetilde{T}^{(r)}_{ij}/(sp_{ij})~~\text{or}~~\exp(-\tcC_{ij}/\eps)/(sp_{ij})$$
only for $(i,j)\in \cS$, where the adjustment factor $sp_{ij}$ ensures the unbiasedness of the estimation.
We then calculate the coupling matrix $\tT^{(r+1)}$ via applying the Sinkhorn-scaling algorithm to the sparse $\tT^{(r)}$ and $\tK^{(r)}$.
Algorithm~\ref{alg:spar-gw} summarizes the \textsc{Spar-GW} algorithm.

{\spacingset{1.25}
\begin{algorithm}[ht]
\caption{\textsc{Spar-GW} algorithm}
\begin{algorithmic}[1]
\State {\bf Input:} Sample distributions $\a, \b$, relation matrices $\C^{X}, \C^{Y}$, ground cost function $\cL$, regularization parameter $\eps$, number of selected elements $s$, number of outer/inner iterations $R$, $H$
\State Construct the sampling probability $\mathbf{P}$ defined by \eqref{eq:pij-ot}  \hfill \textcolor{red}{$O(mn)$}
\State Generate an i.i.d. subsample of size $s$ using $\mathbf{P}$, let $\cS = \{(i_l^\prime, j_l^\prime)\}_{l=1}^s$ be the index set
\State Initialize $\tT^{(0)} = \mathbf{0}_{m\times n}$ and $\widetilde{T}_{ij}^{(0)}=a_i b_j$ if $(i,j)\in\cS$
\State \textbf{For} $r=0 \textbf{ to } R-1$:
\State~~\textbf{Construct a sparse kernel matrix:} \hfill \textcolor{red}{$O(s^2)$}
\begin{itemize}
    \item[a)] Compute $\tcC(\tT^{(r)}) = \sum_{(i,j)\in\cS}\tL_{i\cdot j\cdot} \widetilde{T}_{ij}^{(r)}$ using \eqref{eq:spargw-tL}, and replace its 0's at $\cS$ with $\infty$'s 
    \item[b)] 
    $\tK^{(r)}=
    \begin{cases} 
    \exp(-\frac{\tcC(\tT^{(r)})}{\eps})\odot \tT^{(r)}\oslash (s\mathbf{P}) & \text { if }\cR(\T)=\text{KL}(\T\|\T^{(r)}) \\
    \exp(-\frac{\tcC(\tT^{(r)})}{\eps})\oslash (s\mathbf{P}) & \text { if } \cR(\T)=H(\T)
    \end{cases}$
\end{itemize}
\State~~\textbf{Sinkhorn-scaling (with sparse inputs):} $\tT^{(r+1)} = \textsc{Sinkhorn} (\a, \b, \tK^{(r)}, H)$ \hfill \textcolor{red}{$O(Hs)$}
\State {\bf Output:} $\widehat{\operatorname{GW}} = \sum_{(i, j) \cap (i^\prime, j^\prime) \in \mathcal{S}} \cL\left(C_{i i^{\prime}}^{X}, C_{j j^{\prime}}^{Y}\right) \widetilde{T}^{(R)}_{i j} \widetilde{T}^{(R)}_{i^{\prime} j^{\prime}}$  \hfill \textcolor{red}{$O(s^2)$}
\end{algorithmic}
\label{alg:spar-gw}
\end{algorithm}
}

\textbf{Computational cost.}
Generating the sampling probability matrix $\mathbf{P}$ requires $O(mn)$ time.
In each iteration, calculating $\tcC(\tT^{(r)})$ involves the summation of $s$ matrices, and each of them contains only $s$ non-zero elements, resulting in $O(s^2)$ time. 
For Step 7, the Sinkhorn-scaling algorithm requires $O(Hs)$ time by using sparse matrix multiplications.
Calculating GW distance using the sparse $\tT^{(R)}$ requires $O(s^2)$ operations.
Therefore, for Algorithm~\ref{alg:spar-gw}, its overall time complexity is $O(mn+Rs^2 + RHs + s^2)$, which becomes $O(mn+s^2)$ when $R$ and $H$ are constants, and its memory cost is $O(mn)$. When $m=O(n)$, we obtain the complexity shown in Table~\ref{tab:method-compare}.

\textbf{Applicability for entropic GW distance and fused GW distance.}
As shown in Algorithm~\ref{alg:spar-gw}, our algorithm can approximate the entropic GW distance as well.
Moreover, it is natural to extend the algorithm to approximate the fused Gromov-Wasserstein (FGW) distance~\citep{titouan2019optimal, vayer2020fused}.
In particular, when computing the FGW distance, the cost matrix $\cC(\T)$ takes the direct comparison among the samples into account, and our importance sparsification mechanism is still applicable. 
Details for the modified algorithm are relegated to Appendix.


\section{Theoretical results}\label{sec:theory}

This section shows the convergence and consistency of $\widehat{\text{GW}}$ obtained by Algorithm~\ref{alg:spar-gw} under $\cR(\T)=H(\T)$. 
To ease the conversation, we focus on the case that $m=n$, and the extension to unequal cases is straightforward. 
Let $(\T,\T^\prime)\in\Pi^2(\a,\b)$. We define
\begin{equation*}
    \mathcal{E}(\T,\T') :=  \sideset{}{_{i, i^{\prime}, j, j^{\prime}}}\sum \cL\left(C_{i i^{\prime}}^{X}, C_{j j^{\prime}}^{Y}\right) T_{i j} T_{i^{\prime} j^{\prime}}^\prime,\quad G(\T) := \mathcal{E}(\T,\T)-\sideset{}{_{\T'\in\Pi(\a,\b)}}\min\mathcal{E}(\T,\T').
\end{equation*}
For notation simplicity, we overload $\cE(\T,\T)$ as $\cE(\T)$. 
Following~\cite{kerdoncuff2021sampled}, our goal is to provide a guarantee on the convergence of $G(\T)$, because $\T$ is a stationary point of $\mathcal{E}(\T)$ if and only if $G(\T) = 0$~\citep{reddi2016stochastic}. 
In addition, we define $\barK^{(r)} = \exp(-\cC(\tT^{(r)})/\eps)$, which is the unsampled counterpart to $\tK^{(r)}$ at the $r$-th iteration. 
Consider the following regularity conditions.
\begin{itemize}[noitemsep,topsep=0pt]
	\item[(H.1)] The relation matrices $\C^X, \C^Y$ are symmetric;
	\item[(H.2)] The ground cost is bounded, i.e., $0\le \cL(C_{i i^{\prime}}^{X}, C_{j j^{\prime}}^{Y})\le 2B$ for a constant $B>0$;
	\item[(H.3)] $\|\barK^{(r)}\|_2 \ge n^\alpha/c_1$ for constants $1/2 < \alpha \leq 1$ and $c_1>0$, and the condition number of $\barK^{(r)}$ is positive and bounded by $c_2>0$, for any $r\leq R-1$;
	\item[(H.4)] $p_{ij}\ge {c_3}/{n^2}$ for a constant $c_3>0$;
        \item[(H.5)] $s\ge c_4 n^{3-2\alpha} \log^4(n)$ for constants $c_4=8\log^4(2)/(c_3\log^4(1+\epsilon))$ and $\epsilon>0$.
\end{itemize} 
Conditions (H.1)--(H.3) are natural.
Condition (H.4) indicates that the sampling probabilities could not be too small, requiring $p_{ij}$ to be at the order of $O(1/n^2)$.
The order can always be satisfied by linear interpolating between the proposed sampling probability and the uniform sampling probability.
Such a shrinkage strategy is commonly used in subsampling literature~\citep{ma2015statistical,yu2022optimal}.
Condition (H.5) requires $s$ to be large enough.
For a general case that $\|\barK^{(r)}\|_2=O(n)$, i.e., $\alpha=1$, condition (H.5) can be achieved when $s=O(n^{1+\delta})$ for an arbitrary small $\delta>0$. Such order indicates we only need to compute around $n^2$ elements from the entire tensor that contains $n^4$ elements.

We now provide our main convergence result, whose proofs are provided in Appendix.

\begin{theorem}\label{thm:thm1}
    Under the conditions (H.1)--(H.5), assume that $R\exp\left(-{16}\log^4(n)/{\epsilon^4}\right)\to 0$ for some $\epsilon>0$, and $n>76$. The following bound holds in probability 
    \begin{equation}\label{eq:thm1}
    G({\tT^{(R-1)}})\le \frac{\cE(\tT^{(R)})-\cE(\tT^{(R-1)})}{2} + 6\sqrt{2}\eps {(2+\epsilon)}c_1 c_2 \sqrt{\frac{n^{3-2\alpha}}{c_3 s}}+\eps\log(n)+Bn^2\|\tT^{(R)}-\tT^{(R-1)}\|_F^2.
    \end{equation}
\end{theorem}

Consider the upper bound in Theorem~\ref{thm:thm1}. 
The second term on the right-hand side of~\eqref{eq:thm1} results from importance sparsification. Under the common condition that $\alpha=1$, it tends to zero when $s = O(n^{1+\delta})$ as $n\to \infty$. 
The third term $\eps\log(n)$ is caused by the regularization mechanism, which goes to zero when $\eps = o(\log^{-1}(n))$.
The remaining terms are due to the iterative scheme in Algorithm~\ref{alg:spar-gw}.
Theorem~\ref{thm:thm1} indicates that the estimation error of the proposed estimator decreases when the regularization parameter $\eps$ decreases or the subsample size $s$ increases.
We provide the following corollary to show the consistency of the proposed estimator.

\begin{corollary}\label{cor:cor1}
    Suppose the conditions of Theorem~\ref{thm:thm1} hold with $\alpha=1$. Further suppose Algorithm~\ref{alg:spar-gw} converges with $\|\tT^{(R)}-\tT^{(R-1)}\|_F\le c_5/n^{3/2+\eta}$ for some $c_5,\eta>0$. 
    When $s = O(n^{1+\delta})$ for any $\delta > 0$ and $\eps = o(\log^{-1}(n))$, $G({\tT^{(R-1)}}) \to 0$ in probability, as $n\rightarrow\infty$.
\end{corollary}

The local stationary convergence of Algorithm~\ref{alg:spar-gw} implies $\widetilde{\mathbf{T}}^{(R)}$ and $\widetilde{\mathbf{T}}^{(R-1)}$ will get closer and closer with the increase of $R$. Therefore, for any given $n$, we can set the assumption $||\widetilde{\mathbf{T}}^{(R)} - \widetilde{\mathbf{T}}^{(R-1)}||_F \le c_5 / n^{3/2+\eta}$ as the stopping criterion, which can be naturally satisfied for a relatively large $R$.

\section{Importance sparsification for UGW distance}\label{sec:spar-ugw}

\subsection{Unbalanced GW distance}
In this section, we extend the \textsc{Spar-GW} algorithm to approximate the unbalanced Gromov-Wasserstein (UGW) distance.
Similar to the unbalanced optimal transport (UOT) \citep{liero2016optimal, chizat2018unbalanced, chizat2018scaling, chizat2018interpolating}, UGW distance is able to compare metric-measure spaces endowed with arbitrary positive distributions $\a\in\RR^{m}_{+}, \b\in\RR^{n}_{+}$ \citep{sejourne2021unbalanced, kawano2021classification,luo2022weakly}.
Following the definition in \cite{sejourne2021unbalanced}, UGW distance relaxes the marginal constraints via the quadratic KL-divergence $\mathrm{KL}^{\otimes}(\mu\|\nu) = \mathrm{KL}(\mu\otimes\mu \| \nu\otimes\nu)$, where $\mu\otimes\nu$ is the tensor product measure defined by $d(\mu\otimes\nu)(x,y) = d\mu(x)d\nu(y)$. In particular, UGW distance takes the form
\begin{align*}
    &\operatorname{UGW} \left((\C^X, \a), (\C^Y, \b)\right)\nonumber\\
    :=& \sideset{}{_{\T \in \RR_{+}^{m\times n}}}\min \langle \cL(\C^X, \C^Y)\otimes \T, \T \rangle + \lambda \mathrm{KL}^{\otimes} (\T\mathbf{1}_n \| \a)+\lambda \mathrm{KL}^{\otimes} (\T^\top \mathbf{1}_m \| \b) \nonumber\\
    =& \sideset{}{_{\T \in \RR_{+}^{m\times n}}}\min \langle \cC_{\text{un}}(\T), \T \rangle + \lambda m(\T) \mathrm{KL} (\T\mathbf{1}_n \| \a) + \lambda m(\T) \mathrm{KL}(\T^\top \mathbf{1}_m \| \b ). 
\end{align*}
Here, $m(\T) = \sum_{i,j}T_{ij}$ is the total mass of $\T$, $\cC_{\text{un}}(\T) := \cL(\C^X, \C^Y)\otimes \T + E(\T)$
is the cost matrix with
$E(\T) := \lam \sum_i \log (\sum_j T_{ij}/a_i)\sum_j T_{ij} + \lam \sum_j \log (\sum_i T_{ij}/b_j)\sum_i T_{ij},   
$
and $\lambda > 0$ is the marginal regularization parameter balancing the trade-off between mass transportation and mass variation. 
Note that when $\a\in\Delta^{m-1}$ and $\b\in\Delta^{n-1}$, UGW distance degenerates to the classical GW distance as $\lambda\to\infty$.

\subsection{Proposed algorithm}
To approximate the UGW distance, we update the coupling matrix via proximal gradient algorithm (PGA) by adding a Bregman proximal term \citep{xie2020fast, xu2019gromov, kerdoncuff2021sampled}.
Specifically, $\T$ is updated as
\begin{align}
\T^{(r+1)} = \operatorname{argmin} _{\T \in \RR_{+}^{m\times n}} & \langle \cC_{\text{un}}(\T^{(r)}), \T \rangle + \lambda m(\T^{(r)}) \mathrm{KL} (\T\mathbf{1}_n \| \a ) \notag \\
&+ \lambda m(\T^{(r)}) \mathrm{KL} (\T^\top \mathbf{1}_m \| \b) + \eps m(\T^{(r)}) \mathrm{KL} (\T \| \T^{(r)}). \label{eq:ugw-plan2}
\end{align}

The subproblem~\eqref{eq:ugw-plan2} can be solved using unbalanced Sinkhorn-scaling algorithm \citep{chizat2018scaling, pham2020unbalanced} with the kernel matrix $\K = \exp\{-\cC_{\text{un}}(\T^{(r)}))/(\eps m(\T^{(r)}))\} \odot \T^{(r)}$; see Step~9 in Algorithm~\ref{alg:spar-ugw} for details.
Now the convergent scaling vectors $\u\in\RR^{m}_{+}, \v\in\RR^{n}_{+}$ satisfies that $$(u_i)^{\frac{\lam+\eps}{\lam}} \left(\sum\nolimits_{j} K_{ij} v_j\right) = a_{i} \quad\quad \text{and} \quad\quad  (v_j)^{\frac{\lam+\eps}{\lam}} \left(\sum\nolimits_{i} K_{ij} u_i\right) = b_{j}.$$ 
As a result, it holds that
\begin{equation*}
(u_i)^{\frac{\lam+\eps}{\lam}} K_{ij} v_j \le a_i, \quad u_i K_{ij} (v_j)^{\frac{\lam+\eps}{\lam}} \le b_j \quad\Rightarrow\quad (u_i)^{\frac{2\lam+\eps}{\lam}} K_{ij}^2 (v_j)^{\frac{2\lam+\eps}{\lam}} \le a_i b_j,
\end{equation*}
which follows that $T^\ast_{ij} = u_i K_{ij} v_j \le \left(a_i b_j\right)^{\frac{\lam}{2\lam+\eps}} K_{ij}^{\frac{\eps}{2\lam+\eps}}$.
Such an inequality motivates us to sample with probability
\begin{equation}\label{eq:pij-uot}
p_{ij} =\frac{(a_i b_j)^{\frac{\lam}{2\lam+\eps}} K_{ij}^{\frac{\eps}{2\lam+\eps}}}{\sum_{i,j} (a_i b_j)^{\frac{\lam}{2\lam+\eps}} K_{ij}^{\frac{\eps}{2\lam+\eps}}}, \quad 1\leq i\leq m, \quad 1\leq j\leq n.
\end{equation}
Unfortunately, such a probability involves the kernel matrix $\K$, which requires the unknown coupling matrix $\T$. 
To bypass the obstacle, we propose to replace the unknown $\T$ with the initial value $\tT^{(0)} = \a \b^\top/\sqrt{m(\a)m(\b)}$, where $m(\a)=\sum_{i}a_i$ and $m(\b)=\sum_{j}b_j$ are the total mass of $\a$ and $\b$, respectively.
Algorithm~\ref{alg:spar-ugw} details the proposed \textsc{Spar-UGW} algorithm for approximating UGW distances. 
Note that when $m(\a) = m(\b)$ and $\lambda \to \infty$, \textsc{Spar-UGW} degenerates to \textsc{Spar-GW}. 
Such an observation is consistent with the relationship between GW and UGW.

{\spacingset{1.25}
\begin{algorithm}[ht]
\caption{\textsc{Spar-UGW} algorithm}
\begin{algorithmic}[1]
\State {\bf Input:} Sample distributions $\a, \b$, relation matrices $\C^{X}, \C^{Y}$, ground cost function $\cL$, regularization parameters $\lambda, \eps$, number of selected elements $s$, number of outer/inner iterations $R$, $H$
\State Initialize $\tT^{(0)} = \a \b^\top/\sqrt{m(\a)m(\b)}$
\State $\K = \exp\{-\cC_{\text{un}}(\tT^{(0)})/(\eps m(\tT^{(0)}))\} \odot \tT^{(0)}$  \hfill \textcolor{red}{$O(m^2n^2)$ or $O(mn)$}
\State Construct the sampling probability $\mathbf{P}$ defined by \eqref{eq:pij-uot}  \hfill \textcolor{red}{$O(mn)$}
\State Generate an i.i.d. subsample of size $s$ using $\mathbf{P}$, let $\cS = \{(i_l^\prime, j_l^\prime)\}_{l=1}^s$ be the index set
\State \textbf{For} $r=0 \textbf{ to } R-1$:
\State~~$\bar{\eps} = \eps m(\tT^{(r)}), \bar{\lambda} = \lambda m(\tT^{(r)})$
\State~~\textbf{Construct a sparse kernel matrix:}  \hfill \textcolor{red}{$O(s^2)$}
\begin{itemize}
    \item[a)] Compute the cost matrix $\tcC_{\text{un}}(\tT^{(r)}) = \sum_{(i,j)\in\cS}\tL_{i\cdot j\cdot} \widetilde{T}_{ij}^{(r)} + E(\tT^{(r)})$ using formula~\eqref{eq:spargw-tL} and replace its 0's at $\cS$ with $\infty$'s
    \item[b)] 
    $\tK^{(r)}=\exp(-\frac{\tcC_{\text{un}}(\tT^{(r)})}{\bar{\eps}})\odot \tT^{(r)}\oslash (s\mathbf{P})$
\end{itemize}
\State~~\textbf{Unbalanced Sinkhorn-scaling:}
$\tT^{(r+1)} = \textsc{Sinkhorn}_\textsc{UOT} (\a, \b, \tK^{(r)}, \bar{\lambda}, \bar{\eps}, H)$ \hfill \textcolor{red}{$O(Hs)$}
\begin{itemize}
    \item[a)] Initialize $\u^{(0)} = \mathbf{1}_m, \v^{(0)} = \mathbf{1}_n$ 
    \item[b)] \textbf{For} $h=0 \textbf{ to } H-1$:
    \item[] \quad $\u^{(h+1)} = (\a \oslash (\tK^{(r)} \v^{(h)}))^{\bar{\lambda}/(\bar{\lambda}+\bar{\eps})}$,~~$\v^{(h+1)} = (\b \oslash (\tK^{^{(r)}\top} \u^{(h+1)}))^{\bar{\lambda}/(\bar{\lambda}+\bar{\eps})}$
    \item[c)] $\tT^{(r+1)} = \operatorname{diag}(\u^{(H)}) \tK^{(r)}  \operatorname{diag}(\v^{(H)})$
\end{itemize}
\State~~$\tT^{(r+1)}=\sqrt{m(\tT^{(r)})/m(\tT^{(r+1)})}\cdot \tT^{(r+1)}$
\State {\bf Output:} \small{$\widehat{\operatorname{UGW}} = \langle \cL(\C^X, \C^Y)\otimes \tT^{(R)}, \tT^{(R)} \rangle + \lambda \mathrm{KL}^{\otimes} (\tT^{(R)}\mathbf{1}_n \| \a)+\lambda \mathrm{KL}^{\otimes} (\tT^{(R)\top} \mathbf{1}_m \| \b)$ \hfill \textcolor{red}{$O(s^2)$}}
\end{algorithmic}
\label{alg:spar-ugw}
\end{algorithm}
}

\textbf{Computational cost.}
In Algorithm~\ref{alg:spar-ugw}, although calculating $\K$ in Step~3 requires $O(m^2 n^2)$ time, we only need to calculate it once. 
Moreover, when $\mathcal{L}$ is decomposable, the complexity of calculating $\K$ can be reduced to $O(mn)$ by utilizing the fact that $\T^{(0)}$ is a rank-one matrix.
Therefore, the total time complexity of \textsc{Spar-UGW} is $O(mn + s^2)$ when $R$ and $H$ are constants, and $\mathcal{L}$ is decomposable.

\section{Experiments}\label{sec:experiments}
In this section, we evaluate the performance of 
\textsc{Spar-GW} and its variants in both distance estimation and graph analysis. 

\subsection{Synthetic data analysis}

\subsubsection{Approximation of GW and UGW distances}
We compare the proposed \textsc{Spar-GW} with main competitors including: 
(i) EGW, Algorithm~\ref{alg:gw} with entropic regularization \citep{peyre2016gromov};  
(ii) PGA-GW, Algorithm~\ref{alg:gw} with proximal regularization \citep{xu2019gromov};
(iii) EMD-GW, i.e., EGW with $\eps=0$, but replacing Sinkhorn-scaling algorithm in EGW with the solver for unregularized OT problems \citep{bonneel2011displacement}; 
(iv) S-GWL \citep{xu2019scalable}, adapted for arbitrary ground cost following \cite{kerdoncuff2021sampled}; 
(v) LR-GW, the quadratic approach in \cite{scetbon2022linear};
(vi) SaGroW \citep{kerdoncuff2021sampled}.
Other methods in Table~\ref{tab:method-compare}, i.e., Sliced GW~\citep{titouan2019sliced}, AE~\citep{sato2020fast} and FlowAlign~\citep{le2021flow}, are not included as they fail to approximate the original GW distance.
We adopted the proximal term, i.e., KL-divergence, as $\cR(\T)$ in SaGroW and \textsc{Spar-GW}.
The other choice of regularization term yields similar results.
The regularization parameter $\eps$ is chosen among $\{1, 10^{-1}, 10^{-2}, 10^{-3}\}$ and the result with the smallest distance w.r.t. each method is presented. 
For LR-GW, the non-negative rank of the coupling matrix is set to $\lceil n/20 \rceil$. 
For \textsc{Spar-GW} and \textsc{Spar-UGW}, we set the subsample size $s = 16 n$.
For a fair comparison, we set the subsample size $s' = s^2/n^2$ for SaGroW to ensure that it has the same sampling budget (i.e., samples the same number of elements) to \textsc{Spar-GW} (or \textsc{Spar-UGW}).
Other parameters of the methods mentioned above are set by default.
To take into account the randomness of sampling-based methods, i.e., SaGroW, \textsc{Spar-GW}, and \textsc{Spar-UGW}, their estimations are averaged over ten runs. 
All the experiments are performed on a server with 8-core CPUs and 30GB RAM.

We consider two popular synthetic datasets called \textbf{Moon} following \cite{sejourne2021unbalanced, muzellec2020missing}, and \textbf{Graph} following \cite{xu2019gromov, xu2019scalable}. We also consider two other widely-used datasets including the case where the source and target are distributed in heterogeneous spaces. The results have a similar pattern to those of \textbf{Moon} and are relegated to Appendix.
Specifically, for the \textbf{Moon} dataset, marginals are two Gaussian distributions, $N(n/3, n/20)$ and $N(n/2, n/20)$, supported on $n$ points in $\RR^2$. The source and target supported points are respectively generated from two interleaving half circles by sklearn toolbox \citep{pedregosa2011scikit}. The matrices $\C^X, \C^Y$ are defined using pairwise Euclidean distances in $\RR^2$. 
For the \textbf{Graph} dataset, we first generate one graph with $n$ nodes and power-law degree distribution from NetworkX library \citep{aric2008networkx}, and then generate the other graph by adding extra edges randomly with probability 0.2 on the first one. Their degree distributions are used as two marginals, and the adjacency matrix of each graph is used as $\C^X, \C^Y$. 
Both $\ell_1$ and $\ell_2$ losses are considered for the ground cost. 
LR-GW is only capable of the $\ell_2$ loss, and thus its result w.r.t. the $\ell_1$ loss is omitted. 
To compare the estimation accuracy w.r.t. different methods, we take PGA-GW as a benchmark and calculate the absolute error between its estimation and other estimations of GW distance.

\begin{figure}[ht]
    \centering
    \includegraphics[width=0.99\linewidth]{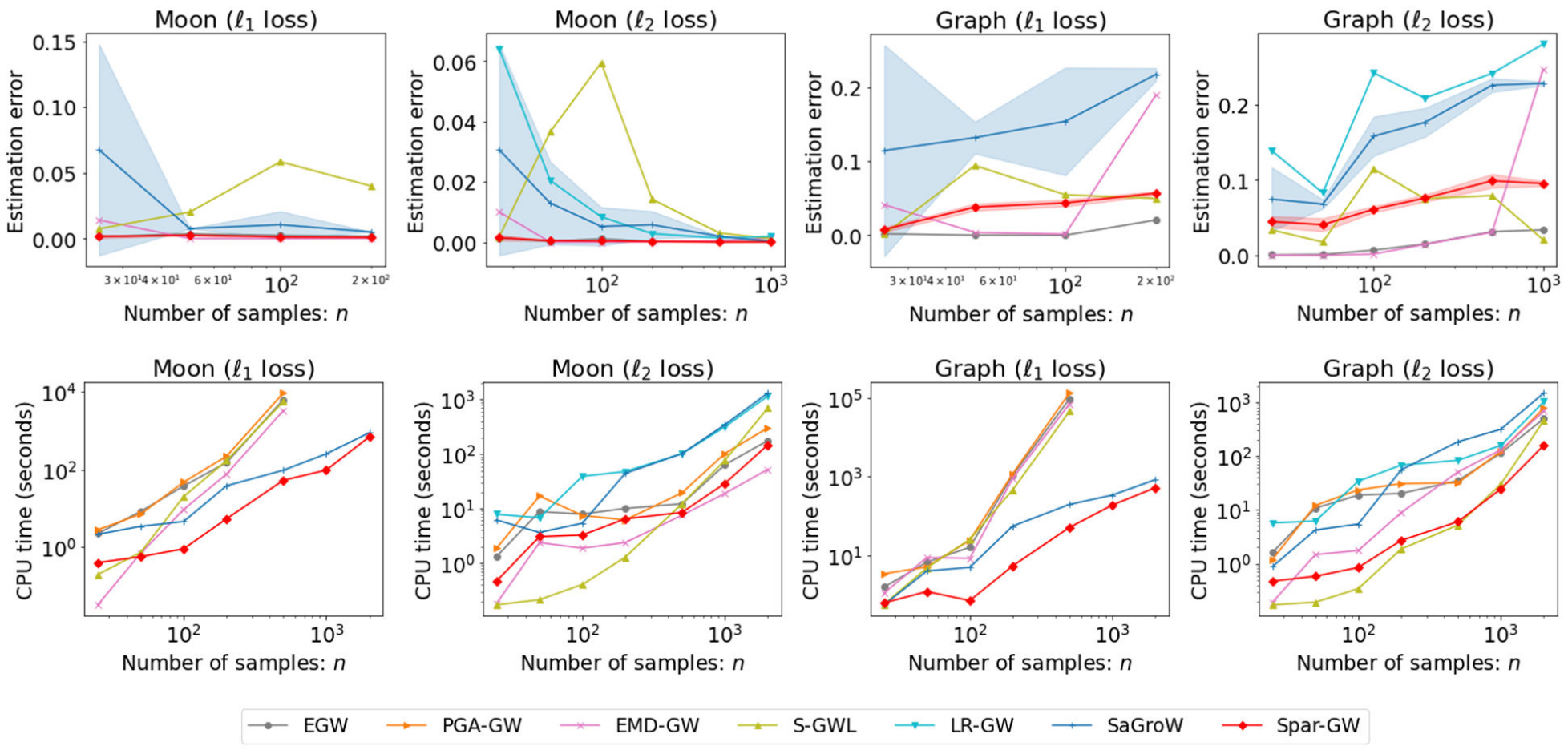}
    \caption{Comparison on estimation error \textbf{(Top)} and CPU time \textbf{(Bottom)}. The mean and standard deviation are reported for sampling-based methods.}
    \label{fig:simu-gw}
\end{figure}

Figure~\ref{fig:simu-gw} shows the estimation error (top row) and the CPU time (bottom row) versus increasing sample size $n$. 
We observe that the proposed \textsc{Spar-GW} method yields almost the smallest estimation error for the \textbf{Moon} dataset and reasonable errors for the \textbf{Graph} dataset. 
Such a difference is because Gaussian distributions in \textbf{Moon} are more concentrated and thus are easier to sketch by subsamples; while the structure of graphs in \textbf{Graph} is more complicated, and the transportation between their degree distributions is also more difficult to approximate by the subsampling technique.
As for computational efficiency, \textsc{Spar-GW} requires less CPU time than most of the competitors, and such an advantage is more obvious for the indecomposable $\ell_1$ loss. 
These observations indicate \textsc{Spar-GW} is capable of dealing with large-scale GW problems with arbitrary ground cost.

\begin{figure}[ht]
    \centering
    \includegraphics[width=0.99\linewidth]{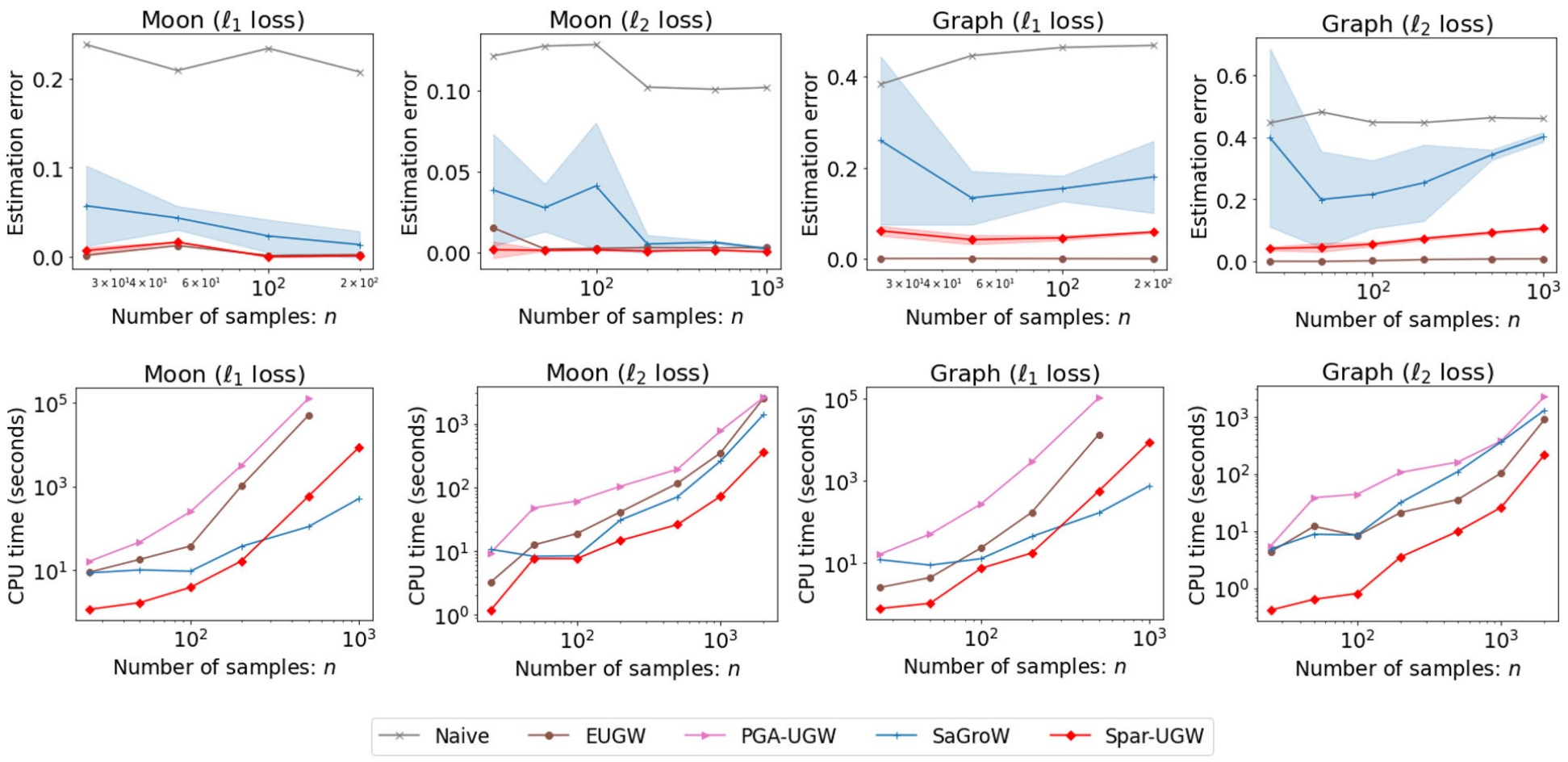}
    \caption{Comparison on estimation error \textbf{(Top)} and CPU time \textbf{(Bottom)} w.r.t. UGW distance. The mean and standard deviation are reported for sampling-based methods.}
    \label{fig:simu-ugw-full}
\end{figure}

For unbalanced problems, we set the total mass of $\a,\b$ to be units and the marginal relaxation parameter to be $\lambda=1$.
We compare \textsc{Spar-UGW} with: 
(i) Naive transport plan $\T = \a \b^\top$;
(ii) EUGW, entropic regularization in \cite{sejourne2021unbalanced};
(iii) PGA-UGW;
(iv) SaGroW, adapted for unbalanced problems.
We calculate the estimation error w.r.t. the PGA-UGW benchmark. 
Other settings are the same as the aforementioned.
From Fig.~\ref{fig:simu-ugw-full}, we observe that \textsc{Spar-UGW} consistently achieves the best accuracy for the former dataset and a relatively small estimation error for the latter one, requiring the least amount of time for the $\ell_2$ loss. Although the computational cost of \textsc{Spar-UGW} becomes more considerable for the indecomposable $\ell_1$ loss, it still computes much faster than the classical EUGW and PGA-UGW methods.

\subsubsection{Sensitivity analysis}
We now show that our method is robust to hyperparameters by analyzing its sensitivity to the subsample size $s$ and the regularization parameter $\eps$. Specifically, for synthetic datasets with fixed sample size $n=200$, the hyperparameters are considered among $s \in \{2^1, 2^2, \ldots, 2^5\}\times n$ and $\eps \in \{5^0, 5^{-1}, \ldots, 5^{-4}\}$. 

\begin{figure}[htbp]
    \centering
    \subfigure[Sensitivity on the estimated GW distance.]{
    \includegraphics[width=0.99\linewidth]{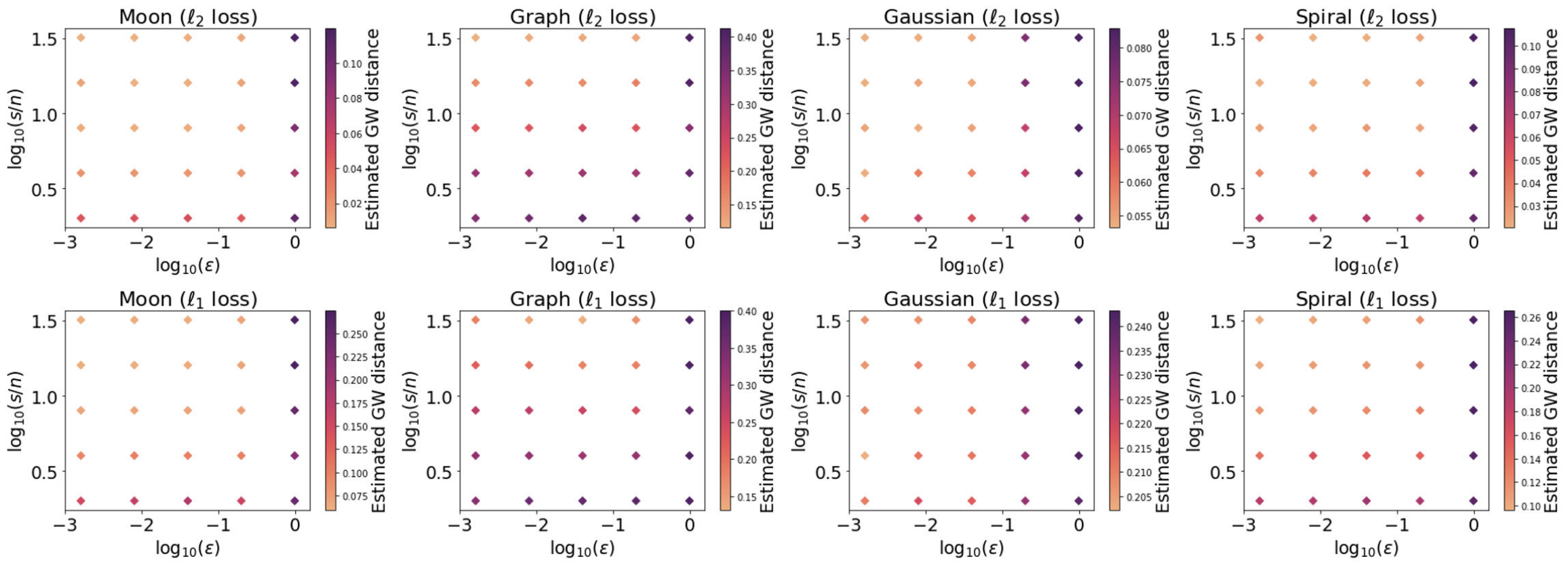}\label{fig:simu-sens-a}
    }
    \subfigure[Sensitivity on the CPU time.]{
    \includegraphics[width=0.99\linewidth]{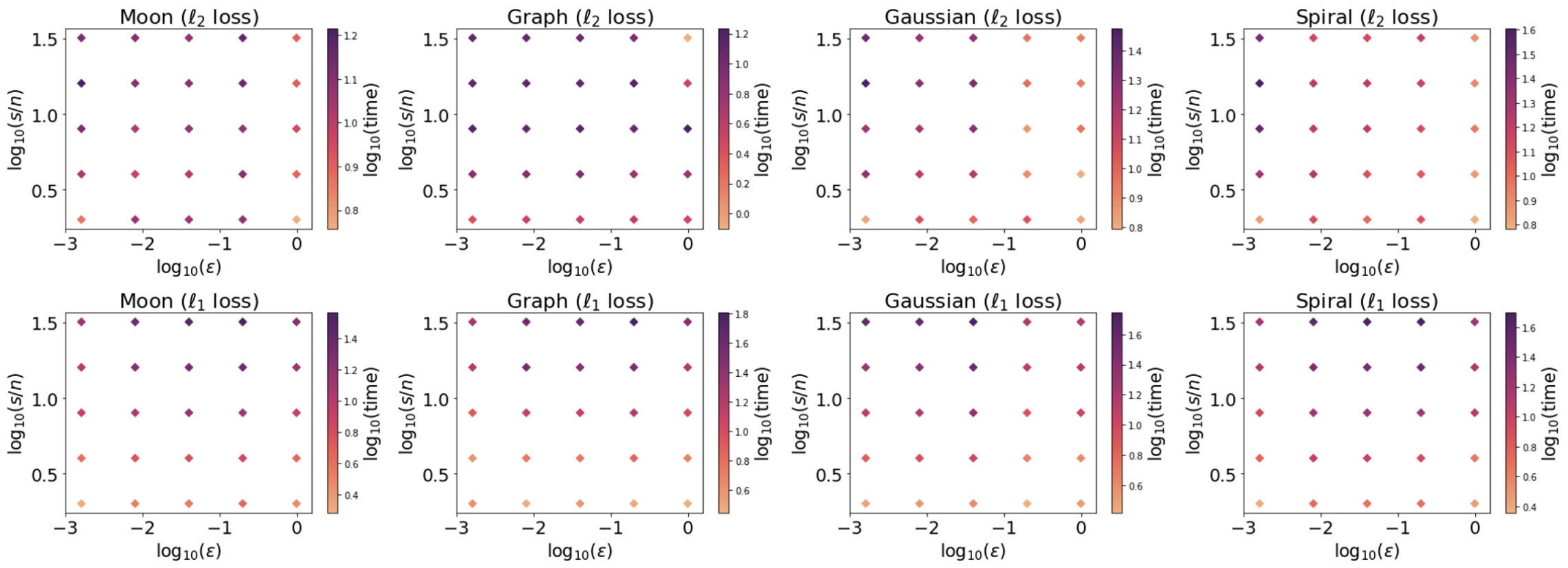}\label{fig:simu-sens-b}
    }
    \caption{Impact of the subsample size $s$ and the regularization parameter $\eps$ for \textsc{Spar-GW} on the GW distance estimation \textbf{(panel (a))} and computational time \textbf{(panel (b))}. The mean over ten runs are reported.}\label{fig:simu-sens}
\end{figure}

From the results in Fig.~\ref{fig:simu-sens}, we find that a large number of selected elements $s$ and/or a small value of $\eps$ is associated with a small GW distance estimation and a long CPU time. Such a finding is consistent with our theoretical results.
We also observe that \textsc{Spar-GW} can yield a satisfactory estimation for a large range of hyperparameters. More precisely, as long as $s=O(n)$ and $\varepsilon$ is not too large, the estimated GW distance is approximately in the same order, which implies \textsc{Spar-GW} is not sensitive to hyperparameters and can cover a wide range of trade-offs between accuracy and speed. 
This observation supports the key assumption that only a few important elements in kernel and coupling matrices are required to approximate the GW distance effectively. 
Moreover, our method is largely free from numerical instability because $\eps$ need not be extremely small, which is in good agreement with the statements in \cite{xie2020fast, xu2019gromov}.

\subsection{Real-world applications}
We consider two applications, graph clustering and graph classification, to demonstrate the effectiveness of our method in applications.
Six widely-used benchmark datasets are considered: BZR, COX2 \citep{sutherland2003spline}, CUNEIFORM \citep{kriege2018recognizing}, SYNTHETIC \citep{feragen2013scalable} with vector node attributes; FIRSTMM\_DB \citep{neumann2013graph} with discrete attributes; and IMDB-B \citep{yanardag2015deep} with no attributes. All these datasets are available in PyTorch Geometric library \citep{fey2019fast}.
Given $N$ graphs, we first compute the pairwise GW distance matrix $\mathbf{D}\in\mathbb{R}^{N \times N}$ and then construct the similarity matrix $\mathbf{S}=\exp(-\mathbf{D}/\gamma)$ for $\gamma>0$.
For methods that can directly extend to approximate the fused GW (FGW) distance, including EGW, PGA-GW, EMD-GW, SaGroW, and \textsc{Spar-GW}, we obtain the pairwise FGW distance matrix when the graphs have attributes.
We set the trade-off parameter $\alpha=0.6$. 
Empirical results show the performance is not sensitive to $\alpha$.
For the graph clustering task, we apply spectral clustering to the similarity matrix. 
We replicate the experiment ten times with different random initialization and assess the clustering performance by average Rand index (RI) \citep{rand1971objective}.
For the classification task, we train a classifier based on kernel SVM using the similarity matrix and test the classifier via nested ten-fold cross-validation following \cite{titouan2019optimal}. 
The performance is evaluated by average classification accuracy.
To examine the effect of different loss functions, we consider both $\ell_1$ loss and $\ell_2$ loss for AE, SaGroW, and \textsc{Spar-GW}. 
Other methods are mainly designed for the decomposable loss, and thus only the $\ell_2$ loss is implemented.
For all methods, $\gamma$ is cross validated within $\{2^{-10}, 2^{-9}, \ldots, 2^{10}\}$. 
Other settings are the same as those in the previous section.

{\spacingset{1.15}
\begin{table}[ht]
  \caption{Comparison on clustering performance w.r.t. RI (\%).}
  \label{tab:res-cluster}
  \centering
  \resizebox{\textwidth}{!}{
  \begin{threeparttable}
  \begin{tabular}{lcccccc}
    \toprule
    Dataset                             & SYNTHETIC      & BZR            & Cuneiform      & COX2           & FIRSTMM\_DB    & IMDB-B         \\
    \# graphs: $N$                      & 300            & 405            & 267            & 467            & 41             & 1000           \\
    Ave. \# nodes: $n$                  & 100.00         & 35.75          & 21.27          & 41.22          & 1377.27        & 19.77          \\
    Subsample size: $s$                 & $2^5\times n$  & $2^3\times n$  & $2^3\times n$  & $2^3\times n$  & $2^7\times n$  & $2^3\times n$  \\
    \midrule
    EGW                                 & \textit{\textbf{100.00}}$_{\pm0.00}$ & 67.18$_{\pm0.44}$ & \textit{\textbf{94.90}}$_{\pm0.08}$ & 64.81$_{\pm0.58}$ & \textit{\textbf{92.51}}$_{\pm0.15}$ &  50.79$_{\pm0.00}$ \\
    S-GWL                               & \textit{\textbf{100.00}}$_{\pm0.00}$ & 66.84$_{\pm0.73}$ & 94.32$_{\pm0.07}$ & 65.02$_{\pm0.23}$ & 81.42$_{\pm0.16}$ & \textbf{51.30}$_{\pm0.00}$   \\
    LR-GW                               & 50.13$_{\pm0.02}$ & 65.34$_{\pm4.31}$ & 26.47$_{\pm0.57}$ & 64.99$_{\pm0.10}$ & 45.93$_{\pm3.14}$ & \textit{\textbf{51.54}}$_{\pm0.01}$ \\
    AE ($\ell_2$ loss)                  & 50.17$_{\pm0.59}$ & 67.04$_{\pm0.00}$ & 82.51$_{\pm3.24}$ & 62.36$_{\pm0.03}$ & 84.63$_{\pm0.00}$ & 50.79$_{\pm0.03}$ \\
    AE ($\ell_1$ loss)                  & 50.17$_{\pm0.59}$ & 67.04$_{\pm0.00}$ & 82.64$_{\pm3.22}$ & 62.36$_{\pm0.00}$ & 84.67$_{\pm0.11}$ & 50.79$_{\pm0.03}$\\
    SaGroW ($\ell_2$ loss)              & 52.41$_{\pm0.00}$ & 67.24$_{\pm0.26}$ & 94.56$_{\pm0.20}$ & \textbf{65.94}$_{\pm0.92}$ & 92.07$_{\pm0.00}$ & 50.45$_{\pm0.00}$  \\
    SaGroW ($\ell_1$ loss)              & 54.15$_{\pm0.19}$ & \textbf{67.33}$_{\pm0.47}$ & 94.54$_{\pm0.14}$ & \textbf{65.97}$_{\pm1.03}$ & 92.09$_{\pm0.35}$ & 50.45$_{\pm0.00}$ \\
    \midrule
    \textsc{Spar-GW} ($\ell_2$ loss)    & \textbf{98.67}$_{\pm0.00}$ & \textit{\textbf{68.22}}$_{\pm0.00}$ & \textbf{94.66}$_{\pm0.05}$ & 65.54$_{\pm0.00}$ & \textbf{92.24}$_{\pm0.30}$ & \textbf{50.82}$_{\pm0.00}$  \\
    \textsc{Spar-GW} ($\ell_1$ loss)    & \textbf{98.67}$_{\pm0.00}$ & \textit{\textbf{68.22}}$_{\pm0.00}$ & \textbf{94.64}$_{\pm0.06}$ & \textit{\textbf{66.29}}$_{\pm1.09}$ & \textbf{92.41}$_{\pm0.33}$ & \textbf{50.82}$_{\pm0.00}$  \\     
    \bottomrule
  \end{tabular}
  \begin{tablenotes}
  \item[*] The top-3 results of each dataset are in bold, and the best result is in italics.
  \end{tablenotes}
  \end{threeparttable}
  }
\end{table}
}

{\spacingset{1.15}
\begin{table}[ht]
  \caption{Comparison on classification performance w.r.t. accuracy (\%).}
  \label{tab:res-class}
  \centering
  \resizebox{\textwidth}{!}{
  \begin{threeparttable}
  \begin{tabular}{lcccccc}
    \toprule
    Dataset                             & SYNTHETIC      & BZR            & Cuneiform      & COX2           & FIRSTMM\_DB    & IMDB-B         \\
    \midrule
    EGW                                 & \textit{\textbf{100.00}}$_{\pm0.00}$    & \textbf{85.92}$_{\pm0.50}$ & \textit{\textbf{25.66}}$_{\pm2.01}$ & \textit{\textbf{80.21}}$_{\pm0.59}$ & \textbf{53.75}$_{\pm2.24}$ & 66.01$_{\pm0.64}$ \\
    S-GWL                               & \textit{\textbf{100.00}}$_{\pm0.00}$    & \textit{\textbf{87.67}}$_{\pm0.41}$ & 9.05$_{\pm1.13}$  & 78.23$_{\pm0.22}$ & 17.50$_{\pm4.03}$  & 64.54$_{\pm0.55}$ \\
    LR-GW                               & 55.83$_{\pm1.46}$ & 79.12$_{\pm0.34}$ & 3.77$_{\pm0.46}$  & 78.06$_{\pm0.20}$ & 15.25$_{\pm3.05}$  & 58.47$_{\pm0.35}$ \\
    AE ($\ell_2$ loss)                  & 43.47$_{\pm1.18}$ & 81.48$_{\pm0.20}$ & 4.90$_{\pm0.67}$ & 78.11$_{\pm0.19}$ & 10.00$_{\pm3.49}$ & 62.98$_{\pm0.40}$ \\
    AE ($\ell_1$ loss)                  & 44.73$_{\pm1.69}$ & 81.65$_{\pm0.34}$ & 5.28$_{\pm0.60}$ & 78.19$_{\pm0.25}$ & 14.50$_{\pm3.72}$ & 63.54$_{\pm0.49}$\\
    SaGroW ($\ell_2$ loss)              & 66.33$_{\pm1.52}$ & 79.47$_{\pm0.32}$ & 17.84$_{\pm1.43}$ & 78.06$_{\pm0.37}$ & 47.50$_{\pm4.10}$ & \textbf{67.40}$_{\pm0.37}$ \\
    SaGroW ($\ell_1$ loss)              & 68.97$_{\pm1.31}$ & 80.17$_{\pm0.76}$ & 16.98$_{\pm1.44}$ & 78.27$_{\pm0.54}$ & 50.00$_{\pm2.79}$  &  \textit{\textbf{67.69}}$_{\pm0.55}$ \\
    \midrule
    \textsc{Spar-GW} ($\ell_2$ loss)    & 98.79$_{\pm0.16}$ & 83.65$_{\pm0.22}$ & \textbf{18.87}$_{\pm0.99}$ & \textbf{78.92}$_{\pm0.11}$ & \textbf{54.25}$_{\pm3.17}$  & 66.70$_{\pm0.46}$ \\
    \textsc{Spar-GW} ($\ell_1$ loss)    & \textbf{99.00}$_{\pm0.22}$   & \textbf{84.19}$_{\pm0.33}$ & \textbf{22.26}$_{\pm1.35}$ & \textbf{78.49}$_{\pm0.69}$ & \textit{\textbf{62.50}}$_{\pm3.81}$ & \textbf{67.00}$_{\pm0.41}$ \\
    \bottomrule
  \end{tabular}
  \begin{tablenotes}
  \item[*] The top-3 results of each dataset are in bold, and the best result is in italics.
  \end{tablenotes}
  \end{threeparttable}
  }
\end{table}
}

Tables~\ref{tab:res-cluster} and \ref{tab:res-class} report the average RI and average classification accuracy with the corresponding standard deviation, respectively.
PGA-GW and EMD-GW are excluded for clarity as their results are similar to EGW. 
Sliced GW and FlowAlign are also not included since they cannot handle graphs.
From Tables~\ref{tab:res-cluster} and \ref{tab:res-class}, we observe the proposed \textsc{Spar-GW} approach is superior or at least comparable to other methods in all cases.
We also observe the \textsc{Spar-GW} with $\ell_1$ cost almost consistently outperforms the one with $\ell_2$ cost.
This observation is consistent with the observation in \cite{kerdoncuff2021sampled}, which stated that the $\ell_1$ cost tends to yield better performance than the $\ell_2$ cost in graphical data analysis.
Such observation also justifies the essence of developing a computational tool that can handle arbitrary ground costs in GW distance approximation.

Considering the CPU time, \textsc{Spar-GW} computes much faster than EGW, S-GWL, and AE when the number of nodes is relatively large. 
Take the FIRSTMM\_DB dataset as an example, in which each graph has an average of 1,377 nodes, the average CPU time for these methods are 414.82s (EGW), 1059.95s (S-GWL), 22.41s (LR-GW), 501.06s/530.12s (AE under $\ell_2$ loss/$\ell_1$ loss), 196.18s/189.57s (SaGroW under $\ell_2$ loss/$\ell_1$ loss), and 82.45s/147.33s (\textsc{Spar-GW} under $\ell_2$ loss/$\ell_1$ loss).
Such results indicate that \textsc{Spar-GW} achieves a decent trade-off between speed and accuracy.

\section{Conclusion}

We developed a novel importance sparsification strategy, achieving the approximation of GW, FGW, and UGW distances in a unified framework with theoretical convergence guarantees.
Experiments show that our \textsc{Spar-GW} method outperforms state-of-the-art approaches in various tasks and attains a decent accuracy-speed trade-off.

We plan to further investigate the theoretical properties of the specific proposed sampling probability, and we are also interested in theoretically deriving the optimal sampling probability.
The proposed importance sparsification mechanism can also be applied to more complex OT problems, for example, the multi-marginal optimal transport problem. Further methodological and theoretical analyses are left to our future work.

\section*{Acknowledgments}
We appreciate the Editor, Associate Editor, and two anonymous reviewers for their constructive comments that helped improve the work.
Mengyu Li is supported by the Outstanding Innovative Talents Cultivation Funded Programs 2021 of Renmin University of China.
The authors would like to acknowledge the support from National Natural Science Foundation of China Grant No.12101606, No.12001042, No.12271522, and Renmin University of China research fund program for young scholars, and Beijing Institute of Technology research fund program for young scholars.
The authors report there are no competing interests to declare.

\newpage
\appendix
\begin{center}
\textbf{\Large{Appendix to ``Efficient Approximation of Gromov-Wasserstein Distance Using Importance Sparsification"}}
\vbox{}
\end{center}

The appendix is structured as follows. In Section~\ref{sec:fgw}, we provide the \textsc{Spar-FGW} algorithm for approximating the fused Gromov-Wasserstein distance. The complete proof of theoretical results is detailed in Section~\ref{sec:proof}. Section~\ref{sec:numerical} presents additional experiments to evaluate the approximation accuracy, time cost, and memory consumption of our method.

\section{\textsc{Spar-FGW} for approximating fused GW distance}\label{sec:fgw}
Wasserstein distance and GW distance focus solely on the feature information and structure information of data, respectively.
For structured data with both the feature and structure information, fused Gromov-Wasserstein (FGW) distance~\citep{titouan2019optimal, vayer2020fused} can be used to capture their geometric properties.

For two sample distributions $\a\in\Delta^{m-1}$ and $\b\in\Delta^{n-1}$, recall that $\mathbf{M}\in\RR^{m \times n}$ stands for the distance matrix between features, and $\C^X, \C^Y$ denote their respective structure relation matrices. Then, FGW distance is defined by trading off the feature and structure relations, as follows,
\begin{align}
\operatorname{FGW} \left((\C^X, \a), (\C^Y, \b)\right) :&= \min _{\T \in \Pi(\a, \b)} \alpha \left \langle \cL(\C^X, \C^Y)\otimes \T, \T \right \rangle + (1-\alpha) \left \langle \mathbf{M}, \T \right \rangle \notag\\
&= \min _{\T \in \Pi(\a, \b)}  \left \langle \alpha \cL(\C^X, \C^Y)\otimes \T + (1-\alpha)\mathbf{M}, \T \right \rangle,  \label{eq:fgw2}
\end{align}
where $\alpha\in[0,1]$ is a trade-off parameter. Wasserstein distance is recovered from the FGW distance as $\alpha\to 0$, and GW distance is recovered as $\alpha\to 1$~\citep{vayer2020fused}.

FGW distance can be similarly approximated using a two-step loop, just by modifying the updated cost matrix $\cC(\T^{(r)})$ in Algorithm~1 of the manuscript to $\cC_{\text{fu}}(\T^{(r)}) := \alpha \cL(\C^X, \C^Y)\otimes \T^{(r)} + (1-\alpha)\mathbf{M}$ illustrated from the formula~\eqref{eq:fgw2}. Thereby, we generalize the \textsc{Spar-GW} algorithm to \textsc{Spar-FGW} algorithm, i.e., Algorithm~\ref{alg:spar-fgw}, in the corresponding way. In particular, given sampled pairs of indices $\cS$, the sparse cost matrix turns to
$\tcC_{\text{fu}}(\tT^{(r)}) := \alpha \sum_{(i,j)\in\cS}\tL_{i\cdot j\cdot} \widetilde{T}_{ij}^{(r)} + (1-\alpha) \widetilde{\mathbf{M}}$, where $\widetilde{\mathbf{M}}$ is defined by
\begin{align*}
\widetilde{M}_{ij} = \begin{cases} M_{ij} & \text { if } (i,j)\in \mathcal{S} \\
0 & \text { otherwise. }\end{cases}  
\end{align*}

\textsc{Spar-FGW} algorithm achieves an interpolation by varying the trade-off parameter $\alpha$: it degenerates to the \textsc{Spar-GW} algorithm as $\alpha$ tends to one, and it tends to approximate the Wasserstein distance as $\alpha$ goes to zero.

{\spacingset{1.25}
\begin{algorithm}[t]
\caption{\textsc{Spar-FGW} algorithm}
\begin{algorithmic}[1]
\State {\bf Input:} Sample distributions $\a, \b$, structure relation matrices $\C^{X}, \C^{Y}$, feature distance matrix $\mathbf{M}$, ground cost function $\cL$, trade-off parameter $\alpha$, regularization parameter $\eps$, number of selected elements $s$, number of outer/inner iterations $R$, $H$
\State Construct the sampling probability $\mathbf{P}=(p_{ij})$ defined by $p_{ij} = \frac{\sqrt{a_i b_j}}{\sum_{i,j} \sqrt{a_i b_j}}$  \hfill \textcolor{red}{$O(mn)$}
\State Generate an i.i.d. subsample of size $s$ using $\mathbf{P}$, let $\cS = \{(i_l^\prime, j_l^\prime)\}_{l=1}^s$ be the index set
\State Initialize $\tT^{(0)} = \mathbf{0}_{m\times n}$ and $\widetilde{T}_{ij}^{(0)}=a_i b_j$ if $(i,j)\in\cS$
\State \textbf{For} $r=0 \textbf{ to } R-1$:
\State~~\textbf{Construct a sparse kernel matrix:}  \hfill \textcolor{red}{$O(s^2)$}
\begin{itemize}
    \item[a)] Compute the cost matrix $\tcC_{\text{fu}}(\tT^{(r)}) = \alpha\sum_{(i,j)\in\cS}\tL_{i\cdot j\cdot} \widetilde{T}_{ij}^{(r)} + (1-\alpha) \widetilde{\mathbf{M}}$ and replace its 0's at $\cS$ with $\infty$'s
    \item[b)] 
    $\tK^{(r)}=
    \begin{cases} 
    \exp(-\frac{\tcC_{\text{fu}}(\tT^{(r)})}{\eps})\odot \tT^{(r)}\oslash (s\mathbf{P}) & \text { if }\cR(\T)=\text{KL}(\T\|\T^{(r)}) \\
    \exp(-\frac{\tcC_{\text{fu}}(\tT^{(r)})}{\eps})\oslash (s\mathbf{P}) & \text { if } \cR(\T)=H(\T)
    \end{cases}$
\end{itemize}
\State~~\textbf{Sinkhorn-scaling (with sparse inputs):} $\tT^{(r+1)} = \textsc{Sinkhorn} (\a, \b, \tK^{(r)}, H)$ \hfill \textcolor{red}{$O(Hs)$}
\State {\bf Output:} \small{$\widehat{\operatorname{FGW}} = \alpha \sum_{(i, j) \cap (i^\prime, j^\prime) \in \mathcal{S}} \cL\left(C_{i i^{\prime}}^{X}, C_{j j^{\prime}}^{Y}\right) \widetilde{T}^{(R)}_{i j} \widetilde{T}^{(R)}_{i^{\prime} j^{\prime}} + (1-\alpha)\sum_{(i, j)\in \mathcal{S}} M_{ij} \widetilde{T}_{ij}^{(R)}$} \hfill \textcolor{red}{$O(s^2)$}
\end{algorithmic}
\label{alg:spar-fgw}
\end{algorithm}
}

\section{Technical details}\label{sec:proof}
Considering probability distributions $\a, \b \in \Delta^{n-1}$ and a cost matrix $\C\in\RR^{n\times n}_{+}$, Sinkhorn-scaling algorithm aims to solve the following optimization problem
\begin{equation}\label{eq:reg-ot}
W_{\eps}(\a,\b) := \min_{\T \in \Pi(\a, \b)} \langle \T,\C\rangle-\eps H(\T).
\end{equation}
By defining the kernel matrix $\K=\exp(-\C/\eps)$ and introducing dual variables $\bm\alpha, \bm\beta \in \RR^{n}$, the dual problem of~\eqref{eq:reg-ot} is
\begin{equation}\label{eq:dual-ot}
	W_{\eps}(\a,\b) = \max_{\boldsymbol{\alpha},\boldsymbol{\beta}\in \RR^{n}} f(\boldsymbol{\alpha},\boldsymbol{\beta}) := \a^\top\boldsymbol{\alpha} + \b^\top\boldsymbol{\beta} - \eps(e^{\boldsymbol{\alpha}/\eps})^\top\K e^{\boldsymbol{\beta}/\eps} + \eps.
\end{equation}
The sparsification counterpart to~\eqref{eq:dual-ot} is
\begin{equation}\label{eq:dual-ot-spar}
	\max_{\boldsymbol{\alpha},\boldsymbol{\beta}\in \RR^{n}} \tilde{f}(\boldsymbol{\alpha},\boldsymbol{\beta}) := \a^\top\boldsymbol{\alpha} + \b^\top\boldsymbol{\beta} - \eps(e^{\boldsymbol{\alpha}/\eps})^\top\tK e^{\boldsymbol{\beta}/\eps} + \eps,
\end{equation}
which replaces $\K$ in~\eqref{eq:dual-ot} with its sparsification $\tK$.

To begin with, we provide the mathematical formula of our subsampling procedure. 
Given the upper bound of the expected number of selected elements $s$ and sampling probabilities $\{p_{ij}\}_{(i,j)\in [n]\times [n]}$ such that $\sum_{i,j}p_{ij} = 1$, we construct the sparsified kernel matrix $\tK$ from $\K$ using the Poisson subsampling framework following the recent work \cite{braverman2021near}. In particular, each element is determined to select or not independently, i.e.,
\begin{align*}
\widetilde{K}_{i j} = \begin{cases} K_{i j}/p^\ast_{i j} & \text { with prob. } p^\ast_{i j}=\min (1, s p_{i j}) \\
0 & \text { otherwise. }\end{cases}  
\end{align*}
Obviously, it holds that $\mathbb{E}(\widetilde{K}_{ij})=K_{ij}$ and $\mathbb{E}(\mathrm{nnz}(\tK)) =\sum_{i,j} p^\ast_{i j} \leq s \sum_{i,j} p_{i j} = s$, where $\mathrm{nnz}(\cdot)$ denotes the number of non-zero elements of a matrix.
Such results indicate the unbiasedness and sparsity of $\tK$.
We now introduce several lemmas.
\begin{lemma}\label{lem:a1}
    Suppose the condition number of $\K$ and $\tK$ are positive and bounded by $c_2>0$ and $c_2^\prime>0$, respectively.
	Let $(\boldsymbol{\alpha}^*,\boldsymbol{\beta}^*)$ be the solution to \eqref{eq:dual-ot}, and $(\bar{\boldsymbol{\alpha}},\bar{\boldsymbol{\beta}})$ be the solution to \eqref{eq:dual-ot-spar}. It follows that
	\begin{equation}\label{eq:lem1}
	|f(\boldsymbol{\alpha}^*,\boldsymbol{\beta}^*)-f(\bar{\boldsymbol{\alpha}},\bar{\boldsymbol{\beta}})|\le \eps\left(c_2+c_2^\prime \left|1-\frac{\|\tK-\K\|_2}{\|\K\|_2}\right|^{-1}\right)\frac{\|\tK-\K\|_2}{\|\K\|_2},
	\end{equation}
    where $\|\cdot\|_2$ denotes the spectral norm (i.e., the largest singular value) of a matrix.
\end{lemma}
\begin{proof}
    First, we establish the following inequality:
    \begin{equation}\label{eq:a16}
    |f(\boldsymbol{\alpha}^*,\boldsymbol{\beta}^*)-f(\bar{\boldsymbol{\alpha}},\bar{\boldsymbol{\beta}})|\le|f(\boldsymbol{\alpha}^*,\boldsymbol{\beta}^*)-\tilde{f}({\boldsymbol{\alpha}}^*,{\boldsymbol{\beta}}^*)|+|\tilde{f}(\bar{\boldsymbol{\alpha}},\bar{\boldsymbol{\beta}})-f(\bar{\boldsymbol{\alpha}},\bar{\boldsymbol{\beta}})|.
    \end{equation}
    By the definitions of $\boldsymbol{\alpha}^*,\boldsymbol{\beta}^*,\bar{\boldsymbol{\alpha}},\bar{\boldsymbol{\beta}}$, it holds that 
    $$\tilde{f}(\bar{\boldsymbol{\alpha}},\bar{\boldsymbol{\beta}}) \ge \tilde{f}(\boldsymbol{\alpha}^*,\boldsymbol{\beta}^*), \quad f(\boldsymbol{\alpha}^*,\boldsymbol{\beta}^*) \ge f(\bar{\boldsymbol{\alpha}},\bar{\boldsymbol{\beta}}).$$ 
    We consider the following two cases:
    \begin{itemize}
        \item[] $\text{Case 1. } f(\boldsymbol{\alpha}^*,\boldsymbol{\beta}^*) \ge \tilde{f}(\bar{\boldsymbol{\alpha}},\bar{\boldsymbol{\beta}})$;
        \item[] $\text{Case 2. } f(\boldsymbol{\alpha}^*,\boldsymbol{\beta}^*) < \tilde{f}(\bar{\boldsymbol{\alpha}},\bar{\boldsymbol{\beta}})$.
    \end{itemize}
    For Case 1, it holds that $0\le f(\boldsymbol{\alpha}^*,\boldsymbol{\beta}^*) - \tilde{f}(\bar{\boldsymbol{\alpha}},\bar{\boldsymbol{\beta}}) \le f(\boldsymbol{\alpha}^*,\boldsymbol{\beta}^*) - \tilde{f}(\boldsymbol{\alpha}^*,\boldsymbol{\beta}^*)$, and thus $|f(\boldsymbol{\alpha}^*,\boldsymbol{\beta}^*) - \tilde{f}(\bar{\boldsymbol{\alpha}},\bar{\boldsymbol{\beta}})| \le |f(\boldsymbol{\alpha}^*,\boldsymbol{\beta}^*) - \tilde{f}(\boldsymbol{\alpha}^*,\boldsymbol{\beta}^*)|$, which leads to~\eqref{eq:a16} by combining the triangle inequality
    $$|f(\boldsymbol{\alpha}^*,\boldsymbol{\beta}^*)-f(\bar{\boldsymbol{\alpha}},\bar{\boldsymbol{\beta}})| \le |f(\boldsymbol{\alpha}^*,\boldsymbol{\beta}^*)-\tilde{f}(\bar{\boldsymbol{\alpha}},\bar{\boldsymbol{\beta}})|+|\tilde{f}(\bar{\boldsymbol{\alpha}},\bar{\boldsymbol{\beta}})-f(\bar{\boldsymbol{\alpha}},\bar{\boldsymbol{\beta}})|.$$
    For Case 2, (i) when $f(\bar{\boldsymbol{\alpha}},\bar{\boldsymbol{\beta}}) \le \tilde{f}({\boldsymbol{\alpha}}^*,{\boldsymbol{\beta}}^*)$, it holds that $0\le \tilde{f}({\boldsymbol{\alpha}}^*,{\boldsymbol{\beta}}^*) - f(\bar{\boldsymbol{\alpha}},\bar{\boldsymbol{\beta}}) \le \tilde{f}(\bar{\boldsymbol{\alpha}},\bar{\boldsymbol{\beta}}) - f(\bar{\boldsymbol{\alpha}},\bar{\boldsymbol{\beta}})$, and thus $| \tilde{f}({\boldsymbol{\alpha}}^*,{\boldsymbol{\beta}}^*) - f(\bar{\boldsymbol{\alpha}},\bar{\boldsymbol{\beta}})| \le |\tilde{f}(\bar{\boldsymbol{\alpha}},\bar{\boldsymbol{\beta}}) - f(\bar{\boldsymbol{\alpha}},\bar{\boldsymbol{\beta}})|$, which leads to~\eqref{eq:a16} by combining the triangle inequality
    $$|f(\boldsymbol{\alpha}^*,\boldsymbol{\beta}^*)-f(\bar{\boldsymbol{\alpha}},\bar{\boldsymbol{\beta}})| \le |f(\boldsymbol{\alpha}^*,\boldsymbol{\beta}^*)-\tilde{f}(\boldsymbol{\alpha}^*,\boldsymbol{\beta}^*)|+|\tilde{f}(\boldsymbol{\alpha}^*,\boldsymbol{\beta}^*)-f(\bar{\boldsymbol{\alpha}},\bar{\boldsymbol{\beta}})|.$$
    (ii) When $f(\bar{\boldsymbol{\alpha}},\bar{\boldsymbol{\beta}}) > \tilde{f}({\boldsymbol{\alpha}}^*,{\boldsymbol{\beta}}^*)$, we have $|f(\boldsymbol{\alpha}^*,\boldsymbol{\beta}^*)-f(\bar{\boldsymbol{\alpha}},\bar{\boldsymbol{\beta}})|\le|f(\boldsymbol{\alpha}^*,\boldsymbol{\beta}^*)-\tilde{f}({\boldsymbol{\alpha}}^*,{\boldsymbol{\beta}}^*)|$; then~\eqref{eq:a16} establishes because $|\tilde{f}(\bar{\boldsymbol{\alpha}},\bar{\boldsymbol{\beta}})-f(\bar{\boldsymbol{\alpha}},\bar{\boldsymbol{\beta}})| \ge 0$.
    
    Consequently, we conclude the inequality~\eqref{eq:a16} by combining Cases 1 and 2. 

    Next, we provide an upper bound for the right-hand side of~\eqref{eq:a16}. 
    Note that both $\K$ and $\tK$ are invertible because their singular values are not zero.
    Simple calculation yields that
	\begin{align}
        \nonumber |f(\boldsymbol{\alpha}^*,\boldsymbol{\beta}^*)-\tilde{f}(\boldsymbol{\alpha}^*,\boldsymbol{\beta}^*)|
	 =&|\eps\langle e^{\boldsymbol{\alpha}^*/\eps},(\tK-\K)e^{\boldsymbol{\beta}^*/\eps}\rangle| \\
	 =&\eps |\textrm{tr}\{(e^{\boldsymbol{\alpha}^*/\eps})^\top(\tK-\K) \K^{-1}\K e^{\boldsymbol{\beta}^*/\eps}\}|. \label{eq:lem1-f1}
	\end{align}
	As $(\boldsymbol{\alpha}^*, \boldsymbol{\beta}^*)$ is the optimal solution to $f(\boldsymbol{\alpha},\boldsymbol{\beta})$, the first order condition implies that 
 $$\textrm{tr}\{\K e^{\boldsymbol{\beta}^*/\eps}(e^{\boldsymbol{\alpha}^*/\eps})^\top\} = (e^{\boldsymbol{\alpha}^*/\eps})^\top\K e^{\boldsymbol{\beta}^*/\eps}=1.$$ 
    Moreover, one can find that
    \begin{align*}
        \|(\tK-\K)\K^{-1}\|_2 \le \|\tK-\K\|_2 / \sigma_{\min}(\K),
    \end{align*}
    where $\sigma_{\min}(\K)$ is the smallest singular value of $\K$.
    For notation simplicity, denote $\mathbf{G}=(\tK-\K) \K^{-1}$ and $\mathbf{H}=\K e^{\boldsymbol{\beta}^*/\eps}(e^{\boldsymbol{\alpha}^*/\eps})^\top$.
	Let $\boldsymbol{h}_j$ be the $j$-th column of $\mathbf{H}$, and $\bm e_j$ be the unit vector with $j$-th element being one. 
	Simple linear algebra yields that
	\begin{align*}
	|\textrm{tr}(\mathbf{G} \mathbf{H})|\le\sum_{j=1}^n {\bm e}_j^\top |\mathbf{G} \bm h_j|\le \sum_{j=1}^n\|\mathbf{G}\|_2\|\bm h_j\|_2, 
	\end{align*}
	where the last equation comes from the Cauchy-Schwarz inequality. Also note that $\mathbf{H}$ is a rank-one matrix; therefore, \eqref{eq:lem1-f1} can be bounded by
    \begin{align}
        \nonumber |f(\boldsymbol{\alpha}^*,\boldsymbol{\beta}^*)-\tilde{f}(\boldsymbol{\alpha}^*,\boldsymbol{\beta}^*)| &\leq \eps \|\tK-\K\|_2 |\textrm{tr}\{ \K e^{\boldsymbol{\beta}^*/\eps}(e^{\boldsymbol{\alpha}^*/\eps})^\top\}|/\sigma_{\min}(\K)  \\
        \nonumber &= \eps \|\tK-\K\|_2/\sigma_{\min}(\K) \\
        &\leq \eps c_2 \|\tK-\K\|_2/\|\K\|_2. \label{eq:a30}
    \end{align}
    
	Using the same procedure, we obtain that
	\begin{align}
	\nonumber |f({\boldsymbol{\bar\alpha}},\boldsymbol{\bar\beta})-\tilde{f}(\boldsymbol{\bar\alpha},\boldsymbol{\bar\beta})|
	\nonumber =&|\eps\langle e^{{\boldsymbol{\bar\alpha}}/\eps},(\tK-\K)e^{{\boldsymbol{\bar\beta}}/\eps}\rangle| \\
        \nonumber =& \eps |\langle e^{{\boldsymbol{\bar\alpha}}/\eps},(\tK-\K) \tK^{-1}\tK e^{{\boldsymbol{\bar\beta}}/\eps}\rangle| \\
        \le & \eps \|(\tK-\K)\tK^{-1}\|_2 |\textrm{tr}\{\tK e^{\bar{\boldsymbol{\beta}}/\eps}(e^{\bar{\boldsymbol{\alpha}}/\eps})^\top\}|. \label{eq:lem1-f2}
	\end{align}
	As $(\bar{\boldsymbol{\alpha}},\bar{\boldsymbol{\beta}})$ is the optimal solution to $\tilde{f}(\boldsymbol{\alpha},\boldsymbol{\beta})$, the first order condition implies that 
    $$\textrm{tr}\{\tK e^{\bar{\boldsymbol{\beta}}/\eps}(e^{\bar{\boldsymbol{\alpha}}/\eps})^\top\} = (e^{\bar{\boldsymbol{\alpha}}/\eps})^\top\tK e^{\bar{\boldsymbol{\beta}}/\eps}=1.$$ 
	Furthermore, simple calculation yields that 
	\begin{align*}
	\|(\tK-\K)\tK^{-1}\|_2 &\le \|\tK-\K\|_2/\sigma_{\min}(\tK)\\
        &\le c_2^\prime \|\tK-\K\|_2/\|\tK\|_2\\
	&= c_2^\prime \frac{\|\tK-\K\|_2}{\|\K\|_2}\frac{\|\K\|_2}{\|\tK\|_2}\\
	&\le c_2^\prime \frac{\|\tK-\K\|_2}{\|\K\|_2}\frac{\|\K\|_2}{|\|\K\|_2-\|\tK-\K\|_2|}\\
	&\le c_2^\prime \frac{\|\tK-\K\|_2}{\|\K\|_2}\left|1-\frac{\|\tK-\K\|_2}{\|\K\|_2}\right|^{-1},
	\end{align*}
	where the second last inequality comes from the triangle inequality. 
	Therefore, \eqref{eq:lem1-f2} satisfies that
	\begin{align}\label{eq:a31}
	|f(\bar{\boldsymbol{\alpha}},\bar{\boldsymbol{\beta}})-\tilde{f}(\bar{\boldsymbol{\alpha}},\bar{\boldsymbol{\beta}})| \le \eps c_2^\prime \frac{\|\tK-\K\|_2}{\|\K\|_2}\left|1-\frac{\|\tK-\K\|_2}{\|\K\|_2}\right|^{-1}.
	\end{align}
	
	Combining \eqref{eq:a16}, \eqref{eq:a30}, and \eqref{eq:a31}, the result follows.
\end{proof}

Let $\tT$ be the solution to the primal of~\eqref{eq:dual-ot-spar}, i.e., $\tT=\diag(e^{\bar{\boldsymbol{\alpha}}/\eps})\tK \diag(e^{\bar{\boldsymbol{\beta}}/\eps})$. Now we show that our subsampling procedure yields a relatively small difference between $W_{\eps}(\a,\b)$ and $\widetilde{W}_{\eps}(\a,\b):=\langle \tT,\C\rangle-\eps H(\tT)$ under some mild conditions. 

\begin{lemma}\label{lem:a2}
	Consider any $\K \in \{\barK^{(0)},\ldots,\barK^{(R-1)}\}$. Under the regularity conditions (H.3)--(H.5), for any $\epsilon>0$ and $n>76$, we have
	\begin{equation}\label{eq:a19}
	\mathbb{P}\left(\frac{\|\tK-\K\|_2}{\|\K\|_2}\ge {2\sqrt{2}(2+\epsilon)}c_1\sqrt{\frac{n^{3-2\alpha}}{c_3 s}}\right)<2\exp\left(-\frac{16}{\epsilon^4} \log^4(n) \right).
	\end{equation}

Moreover, as $n\to\infty$, with probability goes to one, it holds that
\begin{equation}\label{eq:a33}
	|W_{\eps}(\a,\b)-\widetilde{W}_{\eps}(\a,\b)|\le 6\sqrt{2}\eps {(2+\epsilon)}c_1 c_2 \sqrt{\frac{n^{3-2\alpha}}{c_3 s}}\to 0.
\end{equation}
\end{lemma}

\begin{proof}
	Simple calculation yields that
	\begin{align*}
	\mathbb{E}\left(\|\K\|_2^{-1}\widetilde{K}_{ij}\right) &= \|\K\|_2^{-1}K_{ij},\\
	\textrm{Var}\left(\|\K\|_2^{-1}\widetilde{K}_{ij}\right) &< \frac{K_{ij}^2}{p_{ij}^{*}\|\K\|_2^2}\le \frac{1}{p_{ij}^{*}\|\K\|_2^2}\le \frac{n^{2}}{c_3 s\|\K\|_2^2}.
	\end{align*}
	Also note that $\|\K\|_2^{-1}\widetilde{K}_{ij}$ lies between $0$ and $(p_{ij}^*\|\K\|_2)^{-1}$ for any $(i,j)$-th element. 
	Thus, $\|\K\|_2^{-1}\widetilde{K}_{ij}$ takes the value in an interval of length $L$ with
        \begin{align*}
        	L:=\frac{n^{2}}{c_3 s \|\K\|_2} &\le \sqrt{\frac{n^{3-2\alpha}}{2c_3 s}} \times \sqrt{\frac{n^{2}}{c_3 s\|\K\|_2^2}} \times \sqrt{2n}\\
         &\le \left(\frac{\log(1+\epsilon)}{2\log(2n)}\right)^2 \times \sqrt{\frac{n^{2}}{c_3 s\|\K\|_2^2}} \times \sqrt{2n}.
        \end{align*}
	Therefore, according to Theorem 4 in \cite{achlioptas2007fast}, we have
	\begin{equation}
	\mathbb{P}\left(\frac{\|\tK-\K\|_2}{\|\K\|_2}\ge 2(2+\epsilon)\sqrt{\frac{{2n^3}}{c_3 s \|\K\|_2^2}}\right)<2\exp\left(-\frac{16}{\epsilon^4} \log^4(n) \right).\label{eq:a32}
	\end{equation}
	Combining \eqref{eq:a32} with the condition~(i) results in the inequality~\eqref{eq:a19}.

 	From~\eqref{eq:a19}, it is straightforward to see that $\|\K\|_2^{-1}\tK \to \|\K\|_2^{-1}\K$ in probability. Thus, it holds that $c_2^\prime \to c_2$.
	Note that $n^{3-2\alpha}/s\to 0$ as $n\to \infty$, there is no hard to see that $c^\prime \sqrt{n^{3-2\alpha}/s}\le 1/2$ for a constant $c^\prime = 2\sqrt{2}(2+\epsilon)c_1/\sqrt{c_3}$, which implies $(1+|1-c^\prime\sqrt{n^{3-2\alpha}/s}|^{-1})\le 3$, combining which with Lemma~\ref{lem:a1} leads to the result in \eqref{eq:a33}.
\end{proof}

The following results hold from Appendix A.1 in \cite{kerdoncuff2021sampled}.
\begin{lemma}\label{lem:a3}
	Under the regularity conditions (H.1)--(H.2), it holds that
	\begin{align}
		&\|\T-\T'\|_F\le\sqrt{2/n}, \label{eq:a37}\\
		&\nabla \cE(\T)=2\sideset{}{_{i', j'}}\sum \mathbf{L}_{ \cdot i' \cdot j'} T_{i'j'}, \\
		&\|\nabla\mathcal{E}(\T)-\nabla\mathcal{E}(\T')\|_F\le 4Bn^2\|\T-\T'\|_F, \label{eq:a38}\\
		&\mathcal{E}(\T^{(r+1)})\le \cE(\T^{(r)})+\langle \nabla\cE(\T^{(r)}),\T^{(r+1)}-\T^{(r)}\rangle+2Bn^2\|\T^{(r+1)}-\T^{(r)}\|_F^2. \label{eq:a39}
	\end{align}
\end{lemma}

\subsection{Proof of Theorem 1}
Based on the above preliminaries, now we prove our main theoretical results.
\begin{proof}[Proof of Theorem~1]
	Let $\T^{(r+1) *}=\arg\min_{\T^{(r+1)}\in\Pi(\a,\b)}\langle \T^{(r+1)},\nabla\cE(\tT^{(r)})\rangle$, which implies that $\T^{(r+1) *}=\arg\max_{\T^{(r+1)}\in\Pi(\a,\b)}\langle \T^{(r+1)},-\nabla\cE(\tT^{(r)})\rangle$.
	Accordingly, define $\T^{(r+1) *}_\eps$ as the solution to its counterpart with entropy regularization term $\eps H(\T^{(r+1)})$.
	From \eqref{eq:a37}--\eqref{eq:a39} in Lemma~\ref{lem:a3}, we have
	\begin{align}
	\mathcal{E}(\tT^{(r+1)})&\le \cE(\tT^{(r)})+\langle \nabla\cE(\tT^{(r)}),\tT^{(r+1)}-\tT^{(r)}\rangle+2Bn^2\|\tT^{(r+1)}-\tT^{(r)}\|^2_F  \notag\\
	&= \cE(\tT^{(r)})+\langle \nabla\cE(\tT^{(r)}),\tT^{(r+1)}-\T^{(r+1)*}+\T^{(r+1)*}-\tT^{(r)}\rangle \notag\\
	+&2Bn^2\|\tT^{(r+1)}-\tT^{(r)}\|^2_F \notag\\
    &=\cE(\tT^{(r)})+\langle \nabla\cE(\tT^{(r)}),\tT^{(r+1)}-\T^{(r+1) *}\rangle
	+\langle \nabla\cE(\tT^{(r)}),\T^{(r+1) *}-\tT^{(r)}\rangle \notag\\
	+&2Bn^2\|\tT^{(r+1)}-\tT^{(r)}\|^2_F \notag\\
	&\le \cE(\tT^{(r)})+\langle \nabla\cE(\tT^{(r)}),\tT^{(r+1)}-\T^{(r+1) *}_\eps+\T^{(r+1) *}_\eps-\T^{(r+1)*}\rangle  \notag\\
	+&\langle \nabla\cE(\tT^{(r)}),\T^{(r+1)*}-\tT^{(r)}\rangle+2Bn^2\|\tT^{(r+1)}-\tT^{(r)}\|^2_F  \notag\\
	\nonumber &\le \cE(\tT^{(r)})+\langle \nabla\cE(\tT^{(r)}),\tT^{(r+1)}-\T^{(r+1) *}_\eps\rangle +\langle \nabla\cE(\tT^{(r)}),\T^{(r+1) *}_\eps-\T^{(r+1)*}\rangle  \notag\\
	-&2G(\tT^{(r)})+2Bn^2\|\tT^{(r+1)}-\tT^{(r)}\|^2_F  \label{eq:a44}\\
	&\le \cE(\tT^{(r)})+12\sqrt{2}\eps {(2+\epsilon)}c_1 c_2 \sqrt{n^{3-2\alpha}/(c_3 s)}+2\eps\log(n)\notag\\
	-&2G(\tT^{(r)})+2Bn^2\|\tT^{(r+1)}-\tT^{(r)}\|^2_F,  \label{eq:a45}
	\end{align}
	where \eqref{eq:a44} comes from the definition of $G(\cdot)$ and \eqref{eq:a45} comes from Lemma~\ref{lem:a2}.
	
	Therefore, it follows that
	\begin{align*}
	G(\tT^{(r)})&\le 2^{-1}(\cE(\tT^{(r)})-	\mathcal{E}(\tT^{(r+1)}))+6\sqrt{2}\eps {(2+\epsilon)}c_1 c_2 \sqrt{n^{3-2\alpha}/(c_3 s)}\notag\\
	+&\eps\log(n)+Bn^2\|\tT^{(r+1)}-\tT^{(r)}\|^2_F.
	\end{align*}
	The desired results hold by letting $r=R-1$.
\end{proof}

\subsection{Proof of Corollary 1}
\begin{proof}[Proof of Corollary~1]
According to~\eqref{eq:a39} in Lemma~\ref{lem:a3}, one can see that
\begin{align*}\label{eq:a36}
   \cE(\tT^{(R})-\cE(\tT^{(R-1)}) &\le \langle\nabla\cE(\tT^{(R-1)}),\tT^{(R)}-\tT^{(R-1)}\rangle+2Bn^2\|\tT^{(R)}-\tT^{(R-1)}\|_F^2\\
   &\le \|\nabla\cE(\tT^{(R-1)})\|_F \|\tT^{(R)}-\tT^{(R-1)}\|_F + 2Bn^2\|\tT^{(R)}-\tT^{(R-1)}\|_F^2,
\end{align*}
where the last inequality comes from the Cauchy-Schwartz inequality.
Furthermore, simple calculation yields that
\begin{align*}
    \|\nabla\cE(\tT^{(R-1)})\|_F&=\left\|2\sum_{i',j'}\mathbf{L}_{\cdot i'\cdot j'}\widetilde{T}_{i'j'}^{(R-1)}\right\|_F\\
    &=2\sqrt{\sum_{i,j}\left(\sum_{i',j'} L_{i i' j  j'} \widetilde{T}_{i'j'}^{(R-1)}\right)^2}\\
    &=2\sqrt{\sum_{i,j}\left\langle \mathbf{L}_{i\cdot j\cdot },\tT^{(R-1)}\right\rangle ^2}\\
    &\le 2\sqrt{\sum_{i,j}\left(\|\mathbf{L}_{i\cdot j\cdot}\|_F\|\tT^{(R-1)}\|_F \right)^2}\\
    &\le 2\sqrt{\sum_{i,j} 4B^2n^2\|\tT^{(R-1)}\|_F^2}\\
    &\le 4Bn^{3/2},
\end{align*}
where the second last inequality comes from the fact that $\|\tT^{(R-1)}\|_F^2 \le n^{-1}$, which has been shown in Appendix A.1 of \cite{kerdoncuff2021sampled}.
Then, we conclude that
\begin{align*}
    \cE(\tT^{(R)})-\cE(\tT^{(R-1)})\le \frac{4c_5 B}{n^{\eta}}+\frac{2c_5^2 B}{n^{1+2\eta}} \to 0
\end{align*}
under the condition $\|\tT^{(R)}-\tT^{(R-1)}\|_F\le c_5/n^{3/2+\eta}$ for $c_5,\eta>0$, as $n \to \infty$.
Moreover, it holds that $Bn^2\|\tT^{(R)}-\tT^{(R-1)}\|_F^2 \le c_5^2 B/n^{1+2\eta} \to 0$ as $n \to \infty$. Consequently, the consistency of $G({\tT^{(R-1)}})$ follows.

\end{proof}

\section{Additional numerical results}\label{sec:numerical}

\subsection{Approximation of GW distance}
In addition to \textbf{Moon} and \textbf{Graph}, we consider two other widely-used synthetic datasets named \textbf{Gaussian} following \cite{kerdoncuff2021sampled, scetbon2022linear} and \textbf{Spiral} following \cite{titouan2019sliced, weitkamp2020gromov}. 
Source and target are Gaussian distributions that are the same as those of \textbf{Moon} and are supported on $n$ points.
For the former dataset, source and target support points are sampled from mixtures of Gaussians in $\RR^5$ and $\RR^{10}$, respectively. Specifically, the source is mixed of three Gaussians, i.e., $N(\bm{\mu}_s^{(1)}, \mathbf{\Sigma}_s)$, $N(\bm{\mu}_s^{(2)}, \mathbf{\Sigma}_s)$, and $N(\bm{\mu}_s^{(3)}, \mathbf{\Sigma}_s)$, where $\bm{\mu}_s^{(1)}=0\times\bm{1}_5$, $\bm{\mu}_s^{(2)}=\bm{1}_5$, $\bm{\mu}_s^{(3)}=(0,2,2,0,0)$, and $(\Sigma_s)_{ij} = 0.6^{|i-j|}$; the target is a mixture of two Gaussians, i.e., $N(\bm{\mu}_t^{(1)}, \mathbf{\Sigma}_t)$ and $N(\bm{\mu}_t^{(2)}, \mathbf{\Sigma}_t)$, where $\bm{\mu}_t^{(1)}=0.5\times\bm{1}_{10}$, $\bm{\mu}_t^{(2)}=2\times\bm{1}_{10}$, and $\mathbf{\Sigma}_t=\mathbf{I}_{10}$. Here, $\bm{1}_{d}$ denotes the all-ones vector in $\RR^d$, and $\mathbf{I}_{d}$ denotes the identity matrix in $\RR^{d\times d}$.
For the latter one, the source and target are two spirals with noise in $\RR^2$. More precisely, source support points are generated from $\bm{\mu}_s:=(-3\pi\sqrt{r}\cos(3\pi\sqrt{r})+u, 3\pi\sqrt{r}\sin(3\pi\sqrt{r})+u^\prime)-\bm{\mu}_{0}$, where $r,u,u^\prime \stackrel{i.i.d}{\sim} U(0,1)$ and $\bm{\mu}_{0}=(10,10)$ represents translation; target support points are generated from $\bm{\mu}_t := \mathbf{R}\bm{\mu}_s+2\bm{\mu}_{0}$ with a rotation matrix 
\begin{equation*}
\mathbf{R}=\left(\begin{array}{cc}
\cos(\pi/4) & -\sin(\pi/4) \\
\sin(\pi/4) & \cos(\pi/4)
\end{array}\right).
\end{equation*}
For both datasets, the relation matrices $\C^X, \C^Y$ are defined using pairwise Euclidean distances. 

\begin{figure}[ht]
    \centering
    \includegraphics[width=0.99\linewidth]{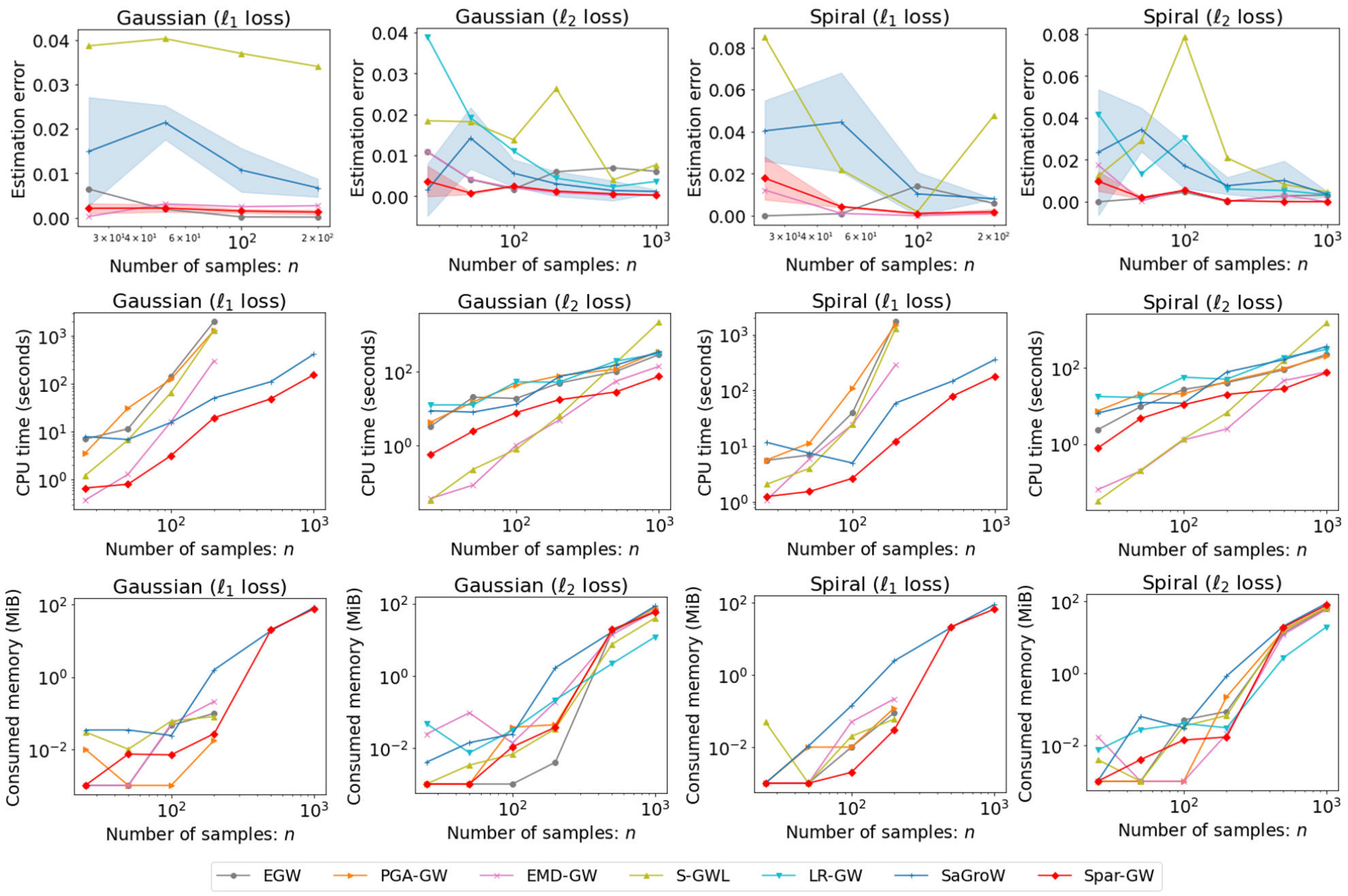}
    \caption{Comparison on estimation error \textbf{(Top)}, CPU time \textbf{(Middle)}, and consumed memory \textbf{(Bottom)} for approximating GW distances. The mean and standard deviation are reported for sampling-based methods.}
    \label{fig:simu-gw2}
\end{figure}

Figure~\ref{fig:simu-gw2} shows the estimation error compared to PGA-GW (top row), the CPU time (middle row), and the consumed memory (bottom row) versus increasing $n$. Memory consumption is measured by the difference between peak and initial memory. 
We observe that our \textsc{Spar-GW} method consistently reaches similar or even better accuracy while being orders of magnitude faster than EGW-based approaches (i.e., EGW, PGA-GW, and EMD-GW), especially for the $\ell_1$ loss.
We also observe that all these methods require similar memory, and this observation is consistent with the fact that these methods have the same order of memory complexity $O(n^2)$.
These observations demonstrate that the proposed \textsc{Spar-GW} method is exceedingly competitive with classical EGW-based methods, requiring much less computational cost.

\subsection{Approximation of fused GW distance}

We evaluate the approximation performance of \textsc{Spar-FGW} algorithm w.r.t. FGW distances, on \textbf{Moon} and \textbf{Graph} datasets in Fig.~\ref{fig:simu-fgw}. The results of \textbf{Gaussian} and \textbf{Spiral} datasets are similar to those of \textbf{Moon}, and thus have been omitted here.

\begin{figure}[htbp]
    \centering
    \includegraphics[width=0.99\linewidth]{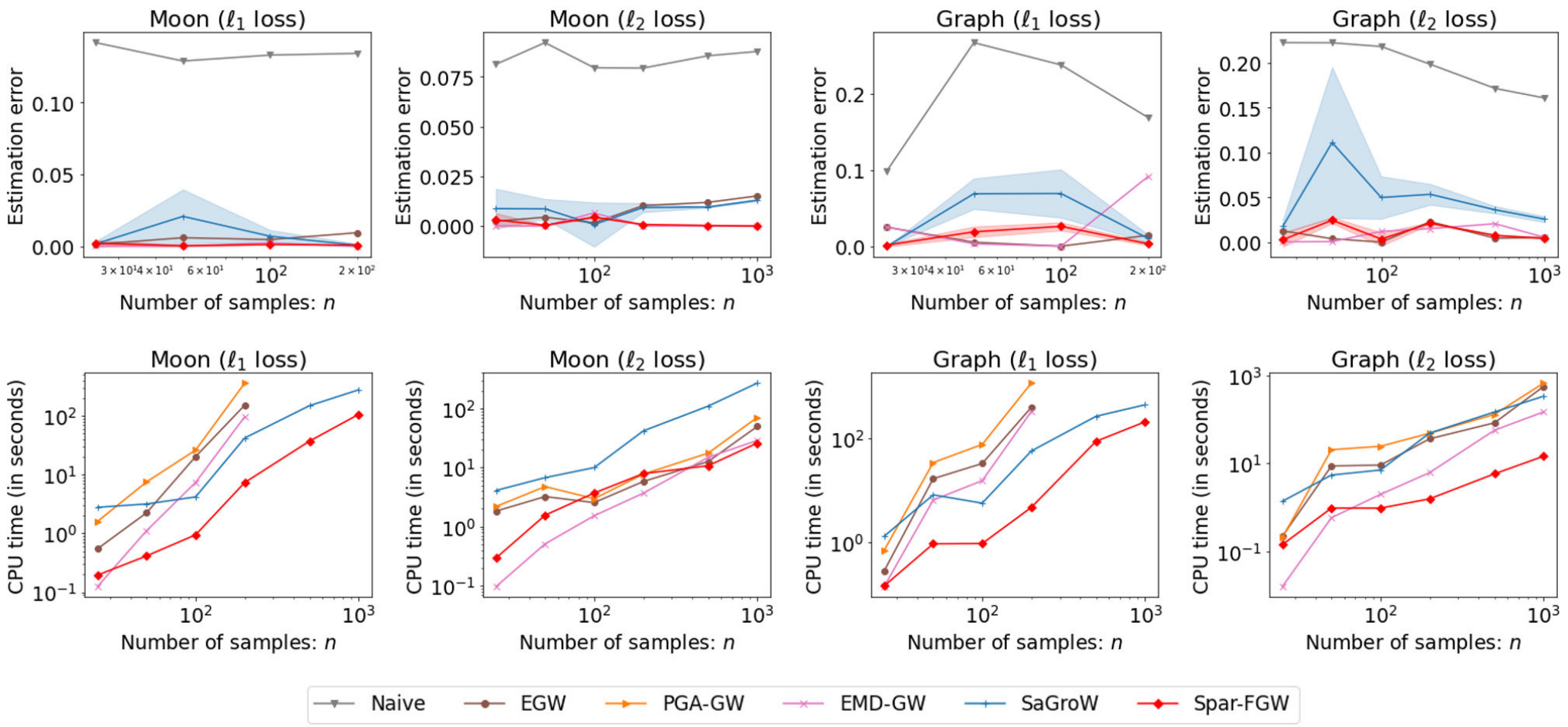}
    \caption{Comparison on estimation error \textbf{(Top)} and CPU time \textbf{(Bottom)} w.r.t. FGW distance. The mean and standard deviation are reported for sampling-based methods.}
    \label{fig:simu-fgw}
\end{figure}

In Fig.~\ref{fig:simu-fgw}, we add another baseline, i.e., the naive transport plan $\T =\a\b^\top$. Other competitors are adapted to approximate the FGW distance for comparison. The source and target attributes are generated from multidimensional Gaussian distributions $N(0\times\bm{1}_5, 10\times\mathbf{I}_5)$ and $N(5\times\bm{1}_5, 10\times\mathbf{I}_5)$, respectively, and the feature distance matrix $\mathbf{M}$ is defined by their pairwise Euclidean distances in $\RR^5$. The trade-off parameter $\alpha$ in~\eqref{eq:fgw2} is set to be 0.6.
Results in Fig.~\ref{fig:simu-fgw} show that \textsc{Spar-FGW} yields the most accurate estimation with the best scalability in most cases.

\newpage
\bibliographystyle{apalike}
\bibliography{ref} 

\end{document}